\documentclass[twoside,11pt]{article}

\usepackage{jmlr2e}
\usepackage[articlestyle]{mpemath}
\usepackage{notation}
\usepackage{subfigure}
\renewcommand{\todo}[1]{}

\setitemize{nolistsep}
\setenumerate{nolistsep}
\setdescription{nolistsep}

\newcommand{\lbl}{\label}
\newcommand{\lblref}[1]{\tag{\ref{#1}}}

\newcommand{\defthm}[3]{
\begin{thm}
#3\label{#2}
\end{thm}
\newcommand{#1}{
\renewcommand{\lbl}{\lblref}
\par \phantomsection \label{#2_proof}
\noindent{\bf Theorem \ref{#2}}.
\emph{ #3 }
\renewcommand{\lbl}{\label}
} \par }

\newcommand{\defprop}[3]{
\begin{prop}
#3 \label{#2}
\end{prop}
\newcommand{#1}{
\renewcommand{\lbl}{\lblref}
\par \phantomsection \label{#2_proof}
\noindent{\bf Proposition \ref{#2}}.
\emph{ #3 }
\renewcommand{\lbl}{\label}
} \par }

\newcommand{\deflem}[3]{
\begin{lem}
#3 \label{#2}
\end{lem}
\newcommand{#1}{
\renewcommand{\lbl}{\lblref}
\par \phantomsection \label{#2_proof}
\noindent{\bf Lemma \ref{#2}}.
\emph{ #3 }
\renewcommand{\lbl}{\label}
} \par }

\ShortHeadings{Global Optimization for Value Function Approximation}{Petrik and Zilberstein}

\title{Global Optimization for Value Function Approximation}
\author{Marek Petrik \email petrik@cs.umass.edu \\
\addr Department of Computer Science \\
University of Massachusetts\\
Amherst, MA 01003, USA
\AND Shlomo Zilberstein \email shlomo@cs.umass.edu \\
\addr Department of Computer Science \\
University of Massachusetts\\
Amherst, MA 01003, USA}
\editor{}

\begin{document}

\maketitle

\begin{abstract}
Existing value function approximation methods have been successfully used in many applications, but they often lack useful a priori error bounds. We propose a new \emph{approximate bilinear programming} formulation of value function approximation, which employs global optimization. The formulation provides strong a priori guarantees on both robust and expected policy loss by minimizing specific norms of the Bellman residual. Solving a bilinear program optimally is NP-hard, but this is unavoidable because the Bellman-residual minimization itself is NP-hard. We describe and analyze both optimal and approximate algorithms for solving bilinear programs. The analysis shows that this algorithm offers a \emph{convergent} generalization of approximate policy iteration. We also briefly analyze the behavior of bilinear programming algorithms under incomplete samples. Finally, we demonstrate that the proposed approach can consistently minimize the Bellman residual on simple benchmark problems.
\end{abstract}



\section{Motivation}

Solving large Markov Decision Problems~(MDPs) is a very useful, but computationally challenging problem addressed widely in the AI literature, particularly in the area of reinforcement learning. It is widely accepted that large MDPs can only be solved approximately. The commonly used approximation methods can be divided into three broad categories: 1) \emph{policy search}, which explores a restricted space of all policies, 2) \emph{approximate dynamic programming}, which searches a restricted space of value functions, and 3) \emph{approximate linear programming}, which approximates the solution using a linear program. While all of these methods have achieved impressive results in many application domains, they have significant limitations.

Policy search methods rely on local search in a restricted policy space. The policy may be represented, for example, as a finite-state controller~\citep{Stanley2004} or as a greedy policy with respect to an approximate value function~\citep{Szita2006}. Policy search methods have achieved impressive results in such domains as Tetris~\citep{Szita2006} and helicopter control~\citep{Abbeel2006}. However, they are notoriously hard to analyze. We are not aware of any established theoretical guarantees regarding the quality of the solution.

Approximate dynamic programming~(ADP) methods iteratively approximate the value function~\citep{Bertsekas1997,Powell2007a,Sutton1998}. They have been extensively analyzed and are the most commonly used methods. However, approximate dynamic programming methods typically do not converge and they only provide weak guarantees of approximation quality. The approximation error bounds are usually expressed in terms of the worst-case approximation of the value function over all policies~\citep{Bertsekas1997}. In addition, most available bounds are with respect to the $L_\infty$ norm, while the algorithms often minimize the $L_2$ norm. While there exist some $L_2$-based bounds~\citep{Munos2003}, they require values that are difficult to obtain.

Approximate linear programming~(ALP) uses a linear program to compute the approximate value function in a particular vector space~\citep{Farias2002}.  ALP has been previously used in a wide variety of settings~\citep{Adelman2004,Farias2004,Guestrin2003}. Although ALP often does not perform as well as ADP, there have been some recent efforts to close the gap~\citep{Petrik2009b}.  ALP has better theoretical properties than ADP and policy search. It is guaranteed to converge and return the closest $L_1$-norm approximation $\tilde\val$ of the optimal value function $v^*$ up to a multiplicative factor. However, the $L_1$ norm must be properly weighted to guarantee a small policy loss, and there is no \emph{reliable} method for selecting appropriate weights~\citep{Farias2002}.

To summarize, the existing reinforcement learning techniques often provide good solutions, but typically require significant domain knowledge~\citep{Powell2007a}. The domain knowledge is needed partly because useful a priori error bounds are not available, as mentioned above. Our goal is to develop a more \emph{reliable} method that is guaranteed to minimize bounds on the policy loss in various settings.

We present new formulations of value function approximation that provably minimize bounds on the policy loss using global optimization. Most of these bounds do not rely on values that are hard to obtain, unlike, for example, approximate linear programming. The focus of the work is on two broad bound minimization approaches: 1) minimizing $L_\infty$ bounds, and 2) minimizing weighted $L_1$ norm bounds on the policy loss. In some sense, the formulations minimize the bounds by unifying policy value-function search methods.

We start with a description of the framework and notation in \aref{sec:framework} and the description of value function approximation in \aref{sec:value_function_approximation}. Then, in \aref{sec:bilinear}, we describe the proposed approximate bilinear programming~(ABP) formulations. Bilinear programs are typically solved using global optimization methods, which we briefly discuss in \aref{sec:solutions}. A drawback of the bilinear formulation is that solving bilinear programs may require exponential time. We also show in \aref{sec:solutions} that this is unavoidable, because minimizing the approximation error bound is in fact NP-hard.

In practice, only sampled versions of ABPs are often solved. While a thorough treatment of sampling is beyond the scope of this paper, we examine the impact of sampling and establish some guarantees in \aref{sec:sampling}. Unlike classical sampling bounds on approximate linear programming, we describe bounds that apply to the worst-case error. \aref{sec:discussion} shows that ABP is related to other approximate dynamic programming methods, such as approximate linear programming and policy iteration. \aref{sec:experiments} demonstrates the applicability of ABP using common reinforcement learning benchmark problems. Technical proofs are provided in the appendix.

The general setting considered in this paper is a restriction of reinforcement learning.  Reinforcement learning methods can use samples without requiring a model of the environment.  The methods we propose can also be based on samples, but they require additional structure. In particular, they require that all or most actions are sampled for every state. Such samples can be easily generated when a model of the environment is available.

\section{Framework and Notation} \label{sec:framework}

This section formally defines the framework and the notation we use. We also define Markov decision processes and the approximation errors involved. Markov decision processes come in many flavors based on the objective function that is optimized. This work focuses  on the infinite horizon discounted MDPs, which are defined as follows.
\begin{defn}[e.g. \citep{Puterman2005}] \label{def:mdp}
A \emph{Markov Decision Process} is a tuple $(\states,\actions,\tran,\rew,\indist)$. Here, $\states$ is a finite set of states, $\actions$ is a finite set of actions, $\tran: \states \times \actions \times \states \mapsto [0,1]$ is the transition function ($\tran(s,a,s')$ is the probability of transiting to state $s'$ from state $s$ given action $a$), and $\rew: \states \times \actions \mapsto \RealPlus$ is a reward function. The initial distribution is: $\indist: \states \mapsto [0,1]$, such that $\sum_{s\in\states} \indist(s) = 1$.
\end{defn}
\noindent The goal is to find a sequence of actions that maximizes $\disc$-discounted discounted cumulative sum of the rewards, also called the \emph{return}. A solution of a Markov decision process is a policy, which is defined as follows.
\begin{defn}
A deterministic stationary \emph{policy} $\pol: \states \mapsto \actions$ assigns an action to each state of the Markov decision process. A stochastic policy \emph{policy} $\pol: \states \times\actions \mapsto [0,1]$. The set of all stochastic stationary policies is denoted as $\policies$ and satisfies $\sum_{a\in\actions} \pol(s,a) = 1$.
\end{defn}
General non-stationary policies may take different actions in states in different time-steps. We limit our treatment to stationary policies, since for infinite-horizon MDPs there exists an optimal \emph{stationary} and \emph{deterministic} policy. We also consider stochastic policies because they are more convenient to use in some settings that we consider.

The transition and reward functions for a given policy are denoted by $P_\pol$ and $r_\pol$. The value function update for a policy $\pol$ is denoted by $\Bell_\pol$, and the Bellman operator is denoted by $L$. That is:
\begin{align*}
\Bell_\pol v &= P_\pol v + r_\pol & \Bell v &= \max_{\pol\in\policies} \Bell_\pol v.
\end{align*}
The optimal value function, denoted $v^*$, satisfies $v^* = \Bell v^*$.

We assume a vector representation of the policy $\pol \in \Real^{|\states||\actions|}$. The variables $\pol$ are defined for all state-action pairs and represent policies. That is $\pol(s,a)$ represents the probability of taking action $a\in\actions$ in state $s\in\states$. The space of all correct (stochastic) policies can be represented using a set of linear equations:
\begin{align*}
\sum_{a \in \actions} \pol(s,a) &= 1 & \forall s \in \states \\
\pol(s,a) &\geq 0 &\forall s\in\states , \forall a\in\actions
\end{align*}
These inequalities can be represented using matrix notation as follows.
\begin{align*}
B \pol &= \one \\
\pol &\geq \zero,
\end{align*}
where the matrix $B: |\states|\times(|\states|\cdot|\actions|)$ is defined as follows.
\[ B(s',(s,a)) = \begin{cases} 1 & s = s' \\
0 & \text{otherwise} \end{cases}. \]
We use $\zero$ and $\one$ to denote vectors of all zeros or ones of the appropriate size respectively. The symbol $\eye$ denotes an identity matrix of the appropriate dimension.

In addition, a policy $\pol$ induces a \emph{state visitation frequency} $u_\pol:\states\rightarrow\Real$, defined as follows:
\[ u_\pol = \invm{\eye - \disc \tran_\pol\tr} \indist.\]
The return of a policy depends on the state-action visitation frequencies and $\indist\tr\val_\pol = \rew\tr u_\pol$. The optimal state-action visitation frequency is $u_{\pol^*}$. \emph{State-action visitation frequency} $u:\states\times\actions\rightarrow\Real$ is defined for all states and actions. Notice the missing subscript. We use $u_a$ to denote the part of $u$ that corresponds to action $a\in\actions$. State-action visitation frequencies must satisfy:
\[ \sum_{a\in\actions} (\eye - \disc\tran)\tr u_a = \indist.\]

To formulate approximate linear and bilinear programs, it is necessary to restrict the value functions so that their Bellman residuals are non-negative (or at least bounded from below). We call such value functions transitive-feasible and define them as follows.
\begin{defn}\label{defn:transitive_feasible}
A value function is \emph{transitive-feasible} when $\val \geq \Bell \val$. The set of transitive-feasible value functions is:
\[ \tf = \{ \val \in \Real^{|\states|} \ss \val \geq \Bell \val\}.\]
Given some $\epsilon \geq 0$, the set of $\epsilon$-\emph{transitive-feasible} value functions is:
\[ \tf(\epsilon) = \{ \val \in \Real^{|\states|} \ss \val \geq \Bell \val - \epsilon \one\}.\]
\end{defn}
Notice that the optimal value function $v^*$ is transitive-feasible. The following lemma summarizes the key property of transitive-feasible value functions:
\deflem{\lemgreatertransitivefeasible}{lem:greater_transitive_feasible}{
Transitive feasible value functions form an upper bound on the optimal value function.
If $\val \in \tf(\epsilon)$ is an $\epsilon$-transitive-feasible value function, then
\[ \val \geq \val^* - \frac{\epsilon}{1-\disc} \one. \]}

\section{Value Function Approximation} \label{sec:value_function_approximation}

This section describes the basic methods for value function approximation. MDPs used in practical applications are often too large for the optimal policy to be computed precisely. In these cases, we first calculate an approximate value function $\tilde\val$ and then take the greedy policy $\pol$ with respect to it. The quality of such a policy can be characterized using its value function $\val_\pol$ in one of the following two main ways.
\begin{defn}[Policy Loss] \label{def:policy_loss}
Let $\pol$ be a policy computed from value function approximation. The \emph{expected policy loss} measures the expected loss of $\pol$, defined as follows:
\begin{equation}
\label{eq:loss_cumulative} \| v^* - \val_\pol   \|_{1,\indist} = \indist\tr\val^* - \indist\tr\val_\pol
\end{equation}
where $\| x \|_{1,c}$ denotes the weighted $L_1$ norm:  $|| x \|_{1,c} =  \sum_{i} |c(i) x(i)|$.\\
The \emph{robust policy loss} measures the worst-case loss of  $\pol$, defined as follows:
\begin{equation}
\label{eq:loss_robust} \| v^* - \val_\pol \|_{\infty} = \max_{s\in\states} |\val^*(s) - \val_\pol(s)|
\end{equation}
\end{defn}

The expected policy loss captures the total loss of discounted reward when following the policy $\pol$ instead of the optimal policy assuming the initial distribution. The robust policy loss ignores the initial distribution and, in some sense, measures the difference for the worst-case initial distribution.

A set of state features is a necessary component of value function approximation. These features must be supplied in advance and must capture the essential structure of the problem. The features are defined by mapping each state $s$ to a vector $\feats(s)$ of features. We denote $\feats_i: \states\rightarrow \Real$ to be a function that maps states to the value of feature $i$:
\[ \feats_i(s) = ( \feats(s) )_i . \]
The desirable properties of the features depend strongly on the algorithm, samples, and attributes of the problem; the tradeoffs are not yet fully understood. The function $\feats_i$ can also be treated as a vector, similarly to the value function $\val$.

Value function approximation methods compute value functions that can be represented using the state features. We call such value functions \emph{representable} and define them below.
\begin{defn} \label{def:representable}
Given a \emph{convex} polyhedral set $\rep \subseteq \Real^{|\states|}$, a value function $\val$ is \emph{representable} (in $\rep$) if $v \in \rep$.
\end{defn}

Many methods that compute a value function based on a given set of features have been developed, such as neural networks and genetic algorithms~\citep{Bertsekas1997}. Most of these methods are extremely hard to analyze, computationally complex, and hard to use. Moreover, these complex methods do not satisfy the convexity assumption in \aref{def:representable}. A simpler, more common, method is \emph{linear value function approximation}. In linear value function approximation, the value function of state $s$ is represented as a linear combination of \emph{nonlinear features} $\feats(s)$. Linear value function approximation is easy to apply and analyze.

Linear value function approximation can be expressed in terms of matrices as follows. Let the matrix $\repm: |\states|\times m$ represent the features for the state-space, where $m$ is the number of features. The rows of the feature matrix $\repm$, also known as the \emph{basis}, correspond to the features of the states $\phi(s)$. The feature matrix can be defined in one of the following two ways:
\[ \repm = \begin{pmatrix} - & \feats(s_1)\tr & - \\  - & \feats(s_2)\tr & - \\  & \vdots & \end{pmatrix} \qquad \qquad
\repm = \begin{pmatrix} | & | & \\
\feats_1 & \feats_2 & \ldots \\
| & | & \end{pmatrix}\]
The value function $v$ is then represented as $\val = \repm x$ and the set of representable functions is $\rep = \cspan{\repm}$.

The  goal of value function approximation is not just to obtain a good value function $\tilde\val$ but a policy with a small policy loss. Unfortunately, the policy loss of a greedy policy, as formulated in \aref{def:policy_loss}, depends non-trivially on the approximate value function $\tilde\val$. Often, the only reliable method of precisely computing the policy loss is to simulate the policy, which can be very costly. Value function approximation methods, therefore, optimize bounds on the policy loss.
\begin{thm}\label{thm:policy_loss_greedy}[Robust Policy Loss, e.g. \citep{Williams1994}]
Let $\pol$ be a greedy policy with respect to a value function $\tilde\val$. Then:
\[\|v^* - v_\pol \|_\infty \leq \frac{2}{1-\disc} \| \tilde\val - \Bell \tilde\val \|_\infty.\]
In addition, if $\tilde\val \in \tf$, the policy loss is minimized for the greedy policy and:
\begin{align*} \|v^* - \val_{\pol} \|_\infty \leq \frac{1}{1-\disc} \|\tilde\val - \Bell \tilde\val\|_\infty.
\end{align*}
\end{thm}

The bounds above ignore the initial distribution and may often be overly conservative. We establish new bounds on the expected policy loss that also consider the initial distribution.
\begin{thm}
\label{thm:bound_dual_policy}[Expected Policy Loss]
Let $\pol$ be a greedy policy with respect to a value function $\tilde\val$ and let the state-action visitation frequencies of $\pol$ be bounded as  $\underbar{u} \leq u_\pol \leq \bar{u}$.
Then:\begin{align*}
\|v^* - v_\pol \|_{1,\indist} &= \indist\tr v^* - \indist\tr \tilde\val +u_\pol\tr\left(  \tilde\val - \Bell \tilde\val \right) \\
&\leq \indist\tr v^* - \indist\tr \tilde\val+ \underbar{u}\tr\negt{\tilde\val - \Bell \tilde\val} + \bar{u}\tr\pos{\tilde\val - \Bell \tilde\val}.
\end{align*}
The state-visitation frequency  $u_\pol$ depends on the initial distribution $\indist$, unlike $v^*$. In addition, when $\tilde\val\in\tf$, the bound is:
\begin{align*} \|v^* - v_\pol \|_{1,\indist} &\leq  - \| v^* - \tilde\val \|_{1,\indist} +  \| \tilde\val - \Bell \tilde\val \|_{1,\bar{u}} \\
\|v^* - v_\pol \|_{1,\indist} &\leq - \| v^* - \tilde\val \|_{1,\indist} +  \frac{1}{1-\disc} \| \tilde\val - \Bell \tilde\val \|_\infty
\end{align*}
\end{thm}
Here we use $\pos{x} = \max\{x,\zero\}$ componentwise.
\begin{proof}
The bound is derived as follows:
\begin{align*}
\indist\tr v^* - \indist\tr \val_\pol &= \indist\tr v^* - \indist\tr \val_\pol + (u_\pol\tr (\eye-\disc\tran_\pol) - \indist\tr) \tilde\val  \\
&= \indist\tr v^* - \rew_\pol\tr u_\pol + (u_\pol\tr (\eye-\disc\tran_\pol) - \indist\tr) \tilde\val \\
&= \indist\tr v^* - \rew_\pol\tr u_\pol + u_\pol\tr (\eye-\disc\tran_\pol) \tilde\val - \indist\tr \tilde\val \\
&= \indist\tr v^* - \indist\tr \tilde\val + u_\pol\tr\left( (\eye-\disc\tran_\pol) \tilde\val - \rew_\pol \right) \\
&= \indist\tr v^* - \indist\tr \tilde\val + u_\pol\tr\left( \tilde\val - \Bell \tilde\val \right).
\end{align*}
We used the fact that $u_\pol\tr(\eye-\disc\tran_\pol) - \indist\tr = \zero$, which follows from the definition of state-action visitation frequencies and that $\val^* \geq \val_\pol$. The inequalities for $\tilde\val\in\tf$ follow from \aref{lem:greater_bellman}, \aref{lem:dual_sum}, and the trivial version of the Holder's inequality:
\begin{align*}
\indist\tr\val^* - \indist\tr\tilde\val &=  - \| v^* - \tilde\val \|_{1,\indist}  \\
u_\pol\tr\left(  \tilde\val - \Bell \tilde\val \right) &\leq  \|u_\pol\|_1 \left\| \tilde\val - \Bell \tilde\val \right\|_\infty  = \frac{1}{1-\disc} \| \tilde\val - \Bell \tilde\val \|_\infty
\end{align*}
\end{proof}

Notice that the bounds in \aref{thm:bound_dual_policy} can be minimized even without knowing the optimal $\val^*$. The optimal value function $\val^*$ is independent of the approximate value function $\tilde\val$ and the greedy policy $\pol$ depends only on $\tilde\val$.

\begin{rem}\label{rem:alp_farias}
The bounds in \aref{thm:bound_dual_policy} generalize the bounds established by \citet[Theorem 1.3]{Farias2002}, which state that whenever $\val \in \tf$:
\[ \| \val^* - \val_\pol \|_{1,u} \leq \frac{1}{1-\disc} \| \val^* - \tilde\val \|_{1,(1-\disc)u}. \]
This bound is a special case of \aref{thm:bound_dual_policy} because $\indist\tr v^* - \indist\tr \tilde\val\leq 0$ and:
\[ \| \tilde\val - \Bell \tilde\val \|_{1,u} \leq  \|\val^* - \tilde\val \|_{1,u} \leq \frac{1}{1-\disc} \| \val^* - \tilde\val \|_{1,(1-\disc)u},\]
from $\val^* \leq \Bell \tilde\val \leq \tilde\val$.
The proof of \aref{thm:bound_dual_policy} also simplifies the proof of Theorem 1.3 in \citep{Farias2002}.
\end{rem}

The methods that we propose require the following standard assumption.
\begin{asm}\label{asm:contains_one}
All multiples of the constant vector $\one$ are representable in $\rep$. That is, for all $k \in \Real$ we have that $k \one \in \rep$.
\end{asm}
Notice that the representation set $\rep$ satisfies \aref{asm:contains_one} when a first column of $\repm$ is $\one$. The impact of including the constant feature is typically negligible because adding a constant to the value function does not change the greedy policy.

Value function approximation algorithms are typically variations of the exact algorithms for solving MDPs. Hence, they can be categorized as approximate value iteration, approximate policy iteration, and approximate linear programming. The ideas of approximate value iteration could be traced to \citet{Bellman1957}, which was followed by many additional research efforts~\citep{Bertsekas1996,Sutton1998,Powell2007a}. Below, we only discuss approximate policy iteration and approximate linear programming, because they are the most closely related to our approach.

\begin{algorithm}[t]
$\pol_0, k \leftarrow $ rand, 1 \;
\While{$\pol_k \neq \pol_{k-1}$}{
    $\tilde\val_k \leftarrow \appr(\pol_{k-1})$ \;
    $\pol_k(s) \leftarrow \arg \max_{a \in \actions} r(s,a) + \disc \sum_{s'\in\states} \tran(s,a,s') \tilde\val_k(s) \quad \forall s \in \states$ \;
    $k \leftarrow k + 1$ \;
}
\caption{Approximate policy iteration, where $\appr(\pol)$ denotes a custom value function approximation for the policy $\pol$.} \label{alg:api}
\end{algorithm}

\emph{Approximate policy iteration}~(API) is summarized in \aref{alg:api}. The function $\appr(\pol)$ denotes the specific method used to approximate the value function for the policy $\pol$.  The two most commonly used methods --- \emph{Bellman residual approximation} and \emph{least-squares approximation}~\citep{Lagoudakis2003} --- minimize the $L_2$ norm of the Bellman residual.

The approximations based on minimizing $L_2$ norm of the Bellman residual are common in practice since they are easy to compute and often lead to good results. Most theoretical analyses of API, however, assume minimization of the $L_\infty$ norm of the Bellman residual:
\begin{equation} \label{eq}
\appr(\pol) \in \arg \min_{\val \in \rep} \left\| (\eye - \disc \tran_\pol) \val - \rew_\pol \right\|_\infty
\end{equation}
\lapi~is shown in \aref{alg:api}, where $\appr(\pol)$ is calculated using the following program:
\begin{mprog}\label{mpr:linfty_min}
\minimize{\phi,v} \phi
\stc (\eye-\disc P_\pol) v + \one \phi \geq r_\pol
\cs -(\eye-\disc P_\pol) v + \one \phi \geq -r_\pol
\cs v \in \rep
\end{mprog}
We are not aware of convergence or divergence proofs of \lapi, and this analysis is beyond the scope of this paper. Theoretically, it is also possible to minimize the $L_1$ norm of the Bellman residual, but we are not aware of any study of such an approximation.

In the above description of API, we assumed that the value function is approximated for all states and actions. This is impossible in practice due to the size of the MDP. Instead, API only relies on a subset of states and actions, provided as samples. API is not guaranteed to converge in general and its analysis is typically in terms of limit behavior. The limit bounds are often very loose. We discuss the performance of API and how it relates to approximate bilinear programming in more detail in \aref{sec:discussion}.

Approximate linear programming --- a method for value function approximation --- is  based on the linear program formulation of exact MDPs:
\begin{mprog} \label{mpr:alp_a}
\minimize{v} \sum_{s\in \states} \tobj(s) v(s)
\stc v(s) - \disc  \sum_{s' \in \states} \tran(s',s,a) v(s') \geq \rew(s,a)  \quad \forall (s,a) \in (\states,\actions)
\end{mprog}
We use $\tmat$ as a shorthand notation for the constraint matrix and $\trew$ for the right-hand side. The value $c$ represents a distribution over the states, usually a uniform one. That is, $\sum_{s\in\states} c(s) = 1$. The linear program \eqref{mpr:alp_a} is often too large to be solved precisely, so it is approximated to get an \emph{approximate linear program} by assuming that $v \in \rep$~\citep{Farias2003}, as follows:
\begin{mprog} \label{mpr:alp} \tag{ALP--$L_1$}
\minimize{\val} c\tr \val
\stc \tmat \val \geq \trew
\cs \val \in \rep
\end{mprog}
The constraint $\val\in\rep$ denotes the approximation. To actually solve this linear program, the value function is represented as $\val = \repm x$. \aref{asm:contains_one} guarantees the feasibility of the ALP. The optimal solution of the ALP, $\tilde\val$, satisfies: $\tilde\val \geq v^*$. Therefore, the objective of \eqref{mpr:alp} represents the minimization of $\| \tilde\val - v^* \|_{1,c}$~\citep{Farias2002}.

Approximate linear programming is guaranteed to converge to a solution and minimize a weighted $L_1$ norm on the solution quality.
\begin{thm}[e.g. \citep{Farias2002}] \label{thm:alp_offline}
Given \aref{asm:contains_one}, let $\tilde\val$ be the solution of the approximate linear program \eqref{mpr:alp}. If $\tobj = \indist$ then
\[ \| \val^* - \tilde\val \|_{1,\indist} \leq \frac{2}{1-\disc} \min_{\val\in\rep} \| \val^* - \val\|_\infty.\]
\end{thm}
The difficulty with the solution of ALP is that it is hard to derive guarantees on the policy loss based on the bounds in terms of the $L_1$ norm; it is possible when the objective function $\tobj$ represents $\bar{u}$, as \aref{rem:alp_farias} shows. In addition, the constant $1/(1-\disc)$ may be very large when $\disc$ is close to 1.

Approximate linear programs are often formulated in terms of samples instead of the full formulation above. The performance guarantees are then based on analyzing the probability that a large number of constraints is violated. It is generally hard to translate the constraint violation bounds to bounds on the quality of the value function and the policy.

\section{Bilinear Program Formulations} \label{sec:bilinear}

This section shows how to formulate value function approximation as a separable bilinear program. Bilinear programs are a generalization of linear programs with an additional bilinear term in the objective function. A separable bilinear program consists of two linear programs with independent constraints and are fairly easy to solve and analyze.
\begin{defn}[Separable Bilinear Program]
A \emph{separable} bilinear program in the normal form is defined as follows:
\begin{mprog} \label{mpr:bilinear} \tag{BP--m}
\minsep{w,x}{y,z}{s_1\tr w + r_1\tr x + x\tr C y + r_2\tr y + s_2\tr z}
\stc A_1 x + B_1 w = b_1 & A_2 y + B_2 z = b_2
\cs w, x \geq \zero & y, z \geq \zero
\end{mprog}
\end{defn}
The objective of the bilinear program \eqref{mpr:bilinear} is denoted as $f(w,x,y,z)$. We separate the variables using a vertical line and the constraints using different columns to emphasize the separable nature of the bilinear program. In this paper, we only use \emph{separable} bilinear programs and refer to them simply as bilinear programs.

We present three different approximate bilinear formulations that minimize the following bounds on the approximate value function.
\begin{enumerate}
    \item \emph{Robust policy loss}: Minimizes $\|\val^* - \val_\pol \|_\infty$ by minimizing the bounds in \aref{thm:policy_loss_greedy}:
     \[ \min_{\pol\in\policies} \|v^* - \val_{\pol} \|_\infty \leq \min_{\val\in\rep} \frac{1}{1-\disc} \|\val - \Bell \val\|_\infty \]
   \item \emph{Expected policy loss}: Minimizes $\| \val^* - \val)\pol \|_{1,\indist}$ by minimizing the bounds in \aref{thm:bound_dual_policy}:
        \begin{align*}
\min_{\pol\in\policies} \|v^* - v_\pol \|_{1,\indist} &\leq \indist\tr v^* +\min_{\val\in\rep} \left( - \indist\tr \tilde\val+  \frac{1}{1-\disc}\| \val - \Bell \val \|_\infty \right)\\
\min_{\pol\in\policies} \|v^* - v_\pol \|_{1,\indist} &\leq \indist\tr v^* + \min_{\val\in\rep} \left(- \indist\tr \tilde\val+  \| \val - \Bell \val \|_{\bar{u}}\right).
        \end{align*}    \item \emph{The sum of $k$ largest errors}: This formulation represents a hybrid between the robust and expected formulations. It is more robust than simply minimizing the expected performance but is not as sensitive to worst-case performance.
\end{enumerate}
The appropriateness of each formulation depends on the particular circumstances of the domain. For example, minimizing robust bounds is advantageous when the initial distribution is not known and the performance must be consistent under all circumstances. On the other hand, minimizing expected bounds on the value function is useful when the initial distribution is known.

In the formulations described below, we initially assume that samples of all states and actions are used. This means that the precise version of the operator $\Bell$ is available. To solve large problems, the number of samples would be much smaller; either simply subsampled or reduced using the structure of the MDP. Reducing the number of constraints in linear programs corresponds to simply removing constraints. In approximate bilinear programs it also reduces the number of some variables, as \aref{sec:sampling} describes.

The formulations below denote the value function approximation generically by $\val \in \rep$. That restricts the value functions to be representable using features. Representable value functions $\val$ can be replaced by a set of variables $x$ as $\val = \repm x$. This reduces the number of variables to the number of features.

\subsection{Robust Policy Loss}

The solution of the robust approximate bilinear program minimizes the $L_\infty$ norm of the Bellman residual $\| \val - \Bell \val \|_\infty$. This minimization can be formulated as follows.
\begin{mprog} \label{mpr:abp_robust} \tag{ABP--$L_\infty$}
\minsep{\pol}{\lambda,\lambda',v}{\pol\tr \lambda + \lambda'}
\stc B \pol = \one        &   \tmat \val - \trew \geq \zero
\cs \pol \geq \zero       &   \lambda + \lambda' \one \geq \tmat \val - \trew
\cs                       &   \lambda, \lambda' \geq \zero
\cs                       &   \val \in \rep
\end{mprog}
All the variables are vectors except $\lambda'$, which is a scalar. The matrix $\tmat$ represents constraints that are identical to the constraints in \eqref{mpr:alp}. The variables $\lambda$ correspond to all state-action pairs. These variables represent the Bellman residuals that are being minimized.
This formulation offers the following guarantees.
\begin{thm}
\label{thm:abp_loss_robust}
Given \aref{asm:contains_one}, any optimal solution $(\tilde{\pol}, \tilde\val, \tilde{\lambda}, \tilde{\lambda'})$ of the approximate bilinear program \eqref{mpr:abp_robust} satisfies:
\begin{align*}
\tilde{\pol}\tr \tilde{\lambda} + \tilde{\lambda'} = \| \Bell \tilde\val - \tilde\val \|_\infty &\leq \min_{v \in \tf \cap \rep} \| \Bell v - v\|_\infty\\
&\leq 2 \min_{v\in\rep} \| \Bell v - v \|_\infty  \\
&\leq 2(1+\disc) \min_{v\in\rep} \| v - v^* \|_\infty.
\end{align*}
Moreover, there exists an optimal solution $\tilde{\pol}$ that is greedy with respect to $\tilde\val$ for which the policy loss is bounded by:
\[ \|\val^* - \val_{\tilde\pol} \|_\infty \leq \frac{2}{1-\disc} \left(\min_{v\in\rep} \| \Bell v - v \|_\infty \right). \]
\end{thm}

It is important to note that the theorem states that solving the approximate bilinear program is equivalent to minimization over \emph{all} representable value functions, not only the transitive-feasible ones. This follows by subtracting a constant vector $\one$ from $\tilde\val$ to balance the lower bounds on the Bellman residual error with the upper ones. This reduces the Bellman residual by $1/2$ without affecting the policy. Finally, note that whenever $v^* \in \rep$, both ABP and ALP will return the optimal value function $\val^*$.

To prove the theorem, we first define the following linear program, which solves for the $L_\infty$ norm of the Bellman update $\Bell_\pol$ for fixed value function $\val$ and policy $\pol$.
\begin{mprog} \label{mpr:error_policy}
f_1(\pol,v) = \minimize{\lambda, \lambda'} \pol\tr \lambda  + \lambda'
\stc \one \lambda' + \lambda \geq \tmat \val - \trew
\cs \lambda \geq \zero
\end{mprog}
The linear program \eqref{mpr:error_policy} corresponds to the bilinear program \eqref{mpr:abp_robust} with a fixed policy $\pol$ and value function $\val$.

\begin{lem}
\label{lem:linprog_linfty}
Let $\val\in\tf$ be a transitive-feasible value function and let $\pol$ be a policy. Then:
\[ f_1(\pol, v) \geq \| \val - \Bell_\pol \val  \|_\infty, \]
with an equality for a \emph{deterministic} policy $\pol$.
\end{lem}
\begin{proof}
The dual of the linear program \eqref{mpr:error_policy} is the following program.
\begin{mprog} \label{mpr:linfty_bound_max}
\maximize{x} x\tr (A v - b)
\stc x \leq \pol
\cs \one\tr x = 1
\cs x \geq \zero
\end{mprog}
Note that replacing $\one\tr x = 1$ by $\one\tr x \leq 1$ preserves the properties of the linear program and would add an additional constraint in \eqref{mpr:abp_robust}: $\lambda' \geq 0$.

First, we show that $f_1(\pol,v) \geq \| \Bell_{\pol} v - v \|_\infty$. Because $v$ is feasible in the approximate bilinear program \eqref{mpr:abp_robust}, $A v - b \geq 0$ and $v \geq \Bell v$ from \aref{lem:greater_bellman}. Let state $s$ be the state in which $t = \| \Bell_{\pol} v - v \|_\infty$ is achieved. That is:
\[ t = v(s) - \sum_{a\in\actions} \pol(s,a) \left( r(s,a) + \sum_{s'\in\states}  \disc P(s',s,a) v(s') \right).\]
Now let $x(s,a) = \pol(s,a)$ for all $a \in \actions$. This is a feasible solution with value $t$, from the stochasticity of the policy and therefore a lower bound on the objective value.

To show the equality for a deterministic policy $\pol$, we show that $f_1(\pol,v) \leq \| \Bell v - v \|_\infty$, using that $\pol \in \{0,1\}$. Then let $x^*$ be an optimal solution of \eqref{mpr:linfty_bound_max}. Define the index of $x^*$ with the largest objective value as:
\[ i \in \arg\max_{ \{ i \ss x^*(i) > 0 \}} (A v - b)(i). \]
Let solution $x'(i) = 1$ and $x'(j) = 0$ for $j \neq i$, which is feasible since $\pol(i) = 1$. In addition:
\[(A v - b)(i) = \|  \Bell_{\pol} v - v \|_\infty.\]
Now $(x^*)\tr (A v - b) \leq (x')\tr (A v - b) = \|  \Bell_{\pol} v - v \|_\infty$, from the fact that $i$ is the index of the largest element of the objective function.
\end{proof}

When the policy $\pol$ is fixed, the approximate bilinear program \eqref{mpr:abp_robust} becomes the following linear program:
\begin{mprog} \label{mpr:abp_fixed_alpha}
f_2(\pol) = \minimize{\lambda,\lambda',v} \pol\tr \lambda + \lambda'
\stc A v - b \geq \zero
\cs \one \lambda + \lambda' \geq A v - b
\cs \lambda \geq \zero
\cs v \in \rep
\end{mprog}
Using \aref{lem:greater_bellman} , this linear program corresponds to:
\[ f_2(\pol) = \min_{v\in\rep\cap \tf} f_1(\pol,v).\]
Then it is easy to show that:
\begin{lem} \label{lem:minimal_error_value}
Given a policy $\pol$, let $\tilde\val$ be an optimal solution of the linear program \eqref{mpr:abp_fixed_alpha}. Then:
\[ f_2(\pol) = \| \Bell_\pol \tilde\val - \tilde\val \|_\infty \leq \min_{\val \in \rep \cap \tf} \| \Bell_\pol \val - \val \|_\infty.\]
\end{lem}

When $v$ is fixed, the approximate bilinear program \eqref{mpr:abp_robust} becomes the following linear program:
\begin{mprog} \label{mpr:error_greedy}
f_3(v) = \minimize{\pol} f_2(\pol,v)
\stc B \pol  = \one
\cs \pol \geq \zero
\end{mprog}
Note that the program is only meaningful if $v$ is transitive-feasible and that the function $f_2$ corresponds to a minimization problem.

\begin{lem} \label{lem:policy_error}
Let $\val \in \rep\cap\tf$ be a transitive-feasible value function. There exists an optimal solution $\tilde\pol$ of the linear program \eqref{mpr:error_greedy} such that:
\begin{enumerate}
    \item $\tilde\pol$ represents a \emph{deterministic} policy
    \item $\Bell_{\tilde\pol} v = \Bell v$
    \item $\| \Bell_{\tilde\pol} v - v \|_\infty = \min_{\pol\in\policies} \| \Bell_{\pol} \val - \val \|_\infty = \| \Bell \val - \val \|_\infty$
\end{enumerate}
\end{lem}
\begin{proof}
The existence of an optimal $\pol$ that corresponds to a deterministic policy follows from \aref{lem:minimal_error_value}, the correspondence between policies and values $\pol$, and the existence of a deterministic greedy policy.

Since $v \in \tf$, we have for some policy $\pol$ that:
\[v \geq \Bell v = \Bell_\pol v \geq \Bell_{\tilde\pol} v. \]
Assuming that $\Bell_{\tilde\pol} < \Bell v$ leads to a contradiction since $\pol$ is also a feasible solution in the linear program \eqref{mpr:error_greedy} and:
\begin{align*}
v - \Bell_{\tilde\pol} &> \val - \Bell \val \\
\| v - \Bell_{\tilde\pol}\|_\infty &> \| \val - \Bell \val \|_\infty.
\end{align*}
This proves the lemma.
\end{proof}

\aref{thm:abp_loss_robust} now easily follows from the lemmas above.
\begin{proof}
Let $\bar{v}$ be a value function with the minimal $\| \Bell \bar{v} - \bar{v} \|_\infty$ feasible in approximate bilinear program \eqref{mpr:abp_robust}, and let $\bar{\pol}$ be a greedy policy with respect to $\bar{v}$. Because $\bar{v} \geq v^*$, as \aref{lem:greater_bellman} shows, we get:
\[ t = \| \Bell \bar{v} - \bar{v} \|_\infty = \| \Bell_{\bar{v}} \bar{v} - \bar{v} \|_\infty. \]
Let $f^*$ be the optimal objective value of \eqref{mpr:abp_robust}. Because both $\bar{v}$ and $\bar{\pol}$ are feasible in \eqref{mpr:abp_robust}, we have that $f^* \leq t$. Now, assume that $\tilde\val$ is an optimal solution of \eqref{mpr:abp_robust} with an objective value $\tilde{f} = \| \Bell \tilde\val - \tilde\val \|_\infty > t$. Then, from \aref{lem:policy_error}, $\tilde{f} > t \geq f^*$, which contradicts the optimality of $\tilde\val$.

To show that the optimal policy is deterministic and greedy, let $\pol^*$ be the optimal policy. Then consider the state $s$ for which $\tilde\pol$ does not define a deterministic greedy action. From the definition of greedy action $\bar{a}$:
\[ (\Bell_{\bar{a}} \tilde\val)(s)  \leq (\Bell_{\tilde\pol} \tilde\val)(s). \]
From the bilinear formulation \eqref{mpr:abp_robust}, it is easy to show that there is an optimal solution such that:
\begin{align*}
(\Bell_{a} \tilde\val)(s)  &\leq \tilde\lambda' + \tilde\lambda(s,a) \\
 \tilde\lambda(s,\bar{a}) &\leq  \tilde\lambda(s,a).
\end{align*}
Then setting $\tilde\pol(s,\bar{a}) = 1$ and all other action probabilities to 0, the difference in the objective value function:
\[ \tilde\lambda(s,\bar{a}) - \sum_{a\in\actions} \tilde\lambda(s,a) \leq 0 . \]
Therefore, the objective function for the deterministic greedy policy does not increase. The remainder of the theorem follows directly from \aref{prop:transitive_nonissue}, \aref{prop:lower_error}, and \aref{lem:optimal_value_error}. The bounds on the policy loss then follow directly from \aref{thm:policy_loss_greedy}.
\end{proof}

\subsection{Expected Policy Loss}

This section describes bilinear programs that minimize expected policy loss for a given initial distribution $\| \val - \Bell \val \|_{1,\indist}$. The initial distribution can be used to derive tighter bounds on the policy loss. We describe two formulations. They respectively minimize an $L_\infty$ and a weighted $L_1$ norm on the Bellman residual.

The expected policy loss can be minimized by solving the following bilinear formulation.
\begin{mprog} \label{mpr:abp_lone_infty} \tag{ABP--$L_1$}
\minsep{\pol}{\lambda,\lambda',v}{\pol\tr \lambda + \lambda' - (1-\disc) \indist\tr \val}
\stc B \pol = \one        &   \tmat \val - \trew \geq \zero
\cs \pol \geq \zero       &   \lambda + \lambda' \one \geq \tmat \val - \trew
\cs                       &   \lambda, \lambda' \geq \zero
\cs                       &   \val \in \rep
\end{mprog}
Notice that this formulation is identical to the bilinear program \eqref{mpr:abp_robust} with the exception of the term $- (1-\disc) \indist\tr \val$.

\begin{thm} \label{thm:abp_loss_expected_infty}
Given \aref{asm:contains_one}, any optimal solution $(\tilde{\pol}, \tilde\val, \tilde{\lambda}, \tilde{\lambda'})$ of the approximate bilinear program \eqref{mpr:abp_lone_infty} satisfies:
\begin{align*}
\frac{1}{1-\disc} \left(\tilde{\pol}\tr \tilde{\lambda} + \tilde{\lambda'} \right) - \indist\tr \tilde\val = \frac{1}{1-\disc} \| \Bell \tilde\val - \tilde\val \|_\infty - \indist\tr \val &\leq \min_{v \in \tf \cap \rep} \left(\frac{1}{1-\disc} \| \Bell v - v\|_\infty  - \indist\tr \val \right)\\
&\leq \min_{v\in\rep} \left(\frac{1}{1-\disc}\| \Bell v - v \|_\infty - \indist\tr \val \right)
\end{align*}
Moreover, there exists an optimal solution $\tilde{\pol}$ that is greedy with respect to $\tilde\val$ for which the policy loss is bounded by:
\[ \|\val^* - \val_{\tilde\pol} \|_{1,\indist} \leq \frac{2}{1-\disc} \left(\min_{v\in\rep} \frac{1}{1-\disc} \| \Bell v - v \|_\infty - \| \val^* - \val \|_{1,\indist} \right). \]
\end{thm}
Notice that the bound in this theorem is tighter than the one in \aref{thm:abp_loss_robust}. In particular, $\| \val^* - \tilde\val \|_{1,\indist} > 0$, unless the solution of the ABP is the optimal value function.
\begin{proof}
The proof of the theorem is almost identical to the proof of \aref{thm:abp_loss_robust} with two main differences. First, the objective function of \eqref{mpr:abp_lone_infty} is insensitive to adding a constant to the value function:
\[ \|(\tilde\val+k\one) - \Bell (\val+k\one) \|_\infty - \indist\tr (\val+k\one)  = \|\tilde\val - \Bell \val \|_\infty - \indist\tr \val. \]
Hence the missing factor 2 when going from minimization over $\tf \cap \rep$ to minimization over $\rep$. The second difference is in the derivation of the bound on the policy loss, which follows directly from \aref{thm:bound_dual_policy}.
\end{proof}

The bilinear program formulation in \eqref{mpr:abp_lone_infty} can be further strengthened when an upper bound on the state-visitation frequencies is available.
\begin{mprog} \label{mpr:abp_lone} \tag{ABP--$U$}
\minsep{\pol}{\lambda,v}{ \pol\tr U \lambda - \indist\tr\val}
\stc B \pol = \one        &   \tmat v - \trew \geq \zero
\cs \pol \geq \zero       &   \lambda \geq \tmat v - \trew
\cs                       &   \lambda \geq \zero
\cs                       &   \val \in \rep
\end{mprog}
Here $U: |\states|\cdot |\actions|\times|\states|\cdot |\actions|$ is a matrix that maps a policy to bounds on state-action visitation frequencies. It must satisfy that:
\[ \pol(s,a) = 0 \Rightarrow (\pol\tr U)(s,a) = 0 \quad \forall s\in\states\; \forall a\in\actions. \]
\begin{rem} \label{rem:constant_u}
One simple option is to have $U$ represent a diagonal matrix of $\bar{u}$, where $\bar{u}$ is the bound for all policies $\pol\in\policies$. That is:
\[ U((s,a),(s',a')) = \begin{cases} \bar{u}(s) & s' = s \\ 0 & \text{otherwise} \end{cases}  \qquad \forall s,s' \in \states \; a,a' \in \actions .\]
\end{rem}
The formal guarantees for this formulation are as follows.
\begin{thm} \label{thm:abp_loss_average}
Given \aref{asm:contains_one} and that for all $\pol\in\policies ~:~ \sum_{a\in\actions}(\pol\tr U)(s,a) \geq u_\pol\tr(s)$, any optimal solution $(\tilde{\pol}, \tilde\val, \tilde{\lambda}, \tilde{\lambda'})$ of the bilinear program \eqref{mpr:abp_lone} then satisfies:
\begin{align*}
\tilde\pol\tr U \tilde\lambda - \indist\tr\tilde\val = \| \tilde\val - \Bell \tilde\val \|_{\bar{u}} - \indist\tr \tilde\val &\leq \min_{v \in \tf \cap \rep} \left( \|\val - \Bell\val \|_{\bar{u}} - \indist\tr\val \right).
\end{align*}
Assuming that $U$ is defined as in \aref{rem:constant_u}, there exists an optimal solution $\tilde{\pol}$ that is greedy with respect to $\tilde\val$ and:
\[ \|\val^* - \val_{\tilde\pol} \|_{1,\indist} \leq \min_{v\in\rep\cap\tf} \left( \| \val - \Bell \val \|_{1,\bar{u}(\val)} - \| \val^* - \val \|_{1,\indist} \right). \]
Here, $\bar{u}(\val)$ represents an upper bound on the state-action visitation frequencies for a policy greedy with respect to value function $\val$.
\end{thm}
Unlike \aref{thm:abp_loss_robust} and \aref{thm:abp_loss_expected_infty}, the bounds in this theorem do not guarantee that the solution quality does not degrade by restricting the value function to be transitive-feasible.

To prove the theorem we first define the following linear program that solves for the $L_1$ norm of the Bellman update $\Bell_\pol$ for a value function $\val$.
\begin{mprog} \label{mpr:error_policy_average}
f_1(\pol,v) = \minimize{\lambda, \lambda'} \pol\tr U \lambda
\stc \one \lambda' + \lambda \geq \tmat \val - \trew
\cs \lambda \geq \zero
\end{mprog}
The linear program \eqref{mpr:error_policy} corresponds to the bilinear program \eqref{mpr:abp_lone} with a fixed policy $\pol$ and value function $\val$.
Notice, in particular, that $\indist\tr\val$ is a constant.

\begin{lem}\label{lem:linprog_lone}
Let value function $\val$ be feasible in the bilinear program \eqref{mpr:abp_lone}, and let $\pol$ be an arbitrary policy. Then:
\[ f_1(\pol, v) \geq \| \Bell_\pol v - v \|_{1,\bar{u}}, \]
with an equality for a \emph{deterministic} policy.
\end{lem}
\begin{proof}
The dual of the linear program \eqref{mpr:error_policy_average} program is the following.
\begin{mprog} \label{mpr:lone_bound_max}
\maximize{x} x\tr (A v - b)
\stc x \leq U\tr \pol
\cs x \geq \zero
\end{mprog}
We have that $f_1(\pol,v) \geq \| \Bell_{\pol} v - v \|_{1,\bar{u}}$ since $x = U\tr \pol$ is a feasible solution.
To show the equality for a deterministic policy $\pol$, let $x^*$ be an optimal solution of linear program \eqref{mpr:lone_bound_max}. Since $\tmat v \geq \trew$ and $U$ is non-negative, an optimal solution satisfies $x = U\tr \pol$. The optimal value of the linear program thus corresponds to the definition of the weighted $L_1$ norm.
\end{proof}
The proof of \aref{thm:abp_loss_average} is similar to the proof of \aref{thm:abp_loss_robust}, but using \aref{thm:bound_dual_policy} instead of \aref{thm:policy_loss_greedy} to bound the policy loss. The existence of a deterministic and greedy optimal solution $\tilde\pol$ follows also like \aref{thm:abp_loss_robust}, omitting $\lambda'$ and weighing $\lambda$ by $\bar{u}$.

\subsection{Hybrid Formulation}

While the robust bilinear formulation \eqref{mpr:abp_robust} guarantees to minimize the robust approximation error it may be overly pessimistic. The bilinear program \eqref{mpr:abp_lone}, on the other hand, optimizes the average performance, but does not provide strong guarantees. It is possible to combine the advantages (and disadvantaged) of these programs using a hybrid formulation. The hybrid formulation minimizes the hybrid norm of the Bellman residual, defined as:
\[ \| x \|_{k,c} = \max_{ \{ y \ss \one\tr y = k, \one \geq y\geq \zero \} } \sum_{i=1}^n y(i) c(i)|x(i)|, \]
where $n$ is the length of vector $x$ and $c\geq \zero$. It is easy to show that this norm represents the $c$-weighted $L_1$ norm of the $k$ largest components of the vector. As such, it is more robust than the plain $L_1$ norm, but is not as sensitive to outliers as the $L_\infty$ norm. Notice that the solution may be fractional when $k\notin\Int$ --- that is, some elements are counted only partially.

The bilinear program that minimizes the hybrid norm is defined as follows.
\begin{mprog} \label{mpr:abp_lonehyb} \tag{ABP--h}
\minsep{\pol}{\lambda,\lambda',\val}{\pol\tr U \lambda + k \lambda'}
\stc B \pol = \one        &   \tmat \val - \trew \geq \zero
\cs \pol \geq \zero       &   \lambda + \lambda' \inv{U} \one \geq \tmat \val - \trew
\cs                       &   \lambda, \lambda' \geq \zero
\cs                       &   \val \in \rep
\end{mprog}
Here $U$ is a matrix that maps a policy to bounds on state-action visitation frequencies, for example, as defined in \aref{rem:constant_u}.

\begin{thm}\label{thm:abp_loss_hybrid}
Given \aref{asm:contains_one} and  $U$ that is defined as in \aref{rem:constant_u}, any optimal solution $(\tilde{\pol}, \tilde\val, \tilde{\lambda}, \tilde{\lambda'})$ of \eqref{mpr:abp_lonehyb} then satisfies:
\begin{align*}
\tilde{\pol}\tr U \tilde{\lambda} + k \tilde{\lambda'} = \| \Bell \tilde\val - \tilde\val \|_{k,\bar{u}(\tilde\pol)} &\leq \min_{\val\in\rep\cap\tf} \| \Bell \tilde\val - \tilde\val \|_{k,\bar{u}(\val)}.
\end{align*}
Here, $\bar{u}(\val)$ represents the upper bound on the state-action visitation frequencies for policy greedy with respect to value function $\val$.
\end{thm}
The implication of these bounds on the policy loss is beyond the scope of this paper, but it is likely that some form of policy loss bounds can be developed.

The proof of the theorem is almost identical to the proof of \aref{thm:abp_loss_robust} lemma. We first define the following linear program, which solves for the required norm of the Bellman update $\Bell_\pol$ for value function $\val$ and policy $\pol$.
\begin{mprog} \label{mpr:error_policy_lonehyb}
f_1(\pol,v) = \minimize{\lambda, \lambda'} \pol\tr U \lambda  + k \lambda'
\stc  \lambda' \inv{U} \one + \lambda \geq \tmat \val - \trew
\cs \lambda,\lambda' \geq \zero
\end{mprog}
The linear program \eqref{mpr:error_policy_lonehyb} corresponds to the bilinear program \eqref{mpr:abp_lonehyb} with a fixed policy $\pol$ and value function $\val$.

\begin{lem} \label{lem:linprog_lonehyb}
Let $\val \in \tf$ be a transitive-feasible value function and let $\pol$ be a policy and $U$ be defined as in \aref{rem:constant_u}. Then:
\[ f_1(\pol, v) \geq \| \val - \Bell_\pol \val  \|_{k,\bar{u}}, \]
with an equality for a \emph{deterministic} policy $\pol$.
\end{lem}
\begin{proof}
The dual of the linear program \eqref{mpr:error_policy} program is the following.
\begin{mprog} \label{mpr:hybrid_bound_max}
\maximize{x} x\tr (\tmat \val - \trew)
\stc x \leq U\tr \pol
\cs \one\tr \invm{U\tr} x \leq k
\cs x \geq \zero
\end{mprog}
First, change the variables in the linear program to $x = U\tr z$ to get:
\begin{mprog} \label{mpr:hybrid_bound_max_cv}
\maximize{z} z\tr U (\tmat \val - \trew)
\stc z \leq \pol
\cs \one\tr z \leq k
\cs z \geq \zero
\end{mprog}
using the fact that $U$ is diagonal and positive.

The norm$\| \Bell_{\pol} v - v \|_{k,c}$ can be expressed as the following linear program:
\begin{mprog} \label{mpr:norm_def}
\maximize{y} y\tr X U (\tmat \val - \trew)
\stc y \leq \one
\cs \one\tr y \leq k
\cs y \geq \zero
\end{mprog}
Here, the matrix $X:|\states|\times |\states|\cdot|\actions|$ selects the subsets of the Bellman residuals that correspond the the policy as defined:
\[ X(s,(s',a')) = \begin{cases} \pol(s',a') & \text{when } s = s' \\ 0 & \text{otherwise} \end{cases}. \]
It is easy to shows that $v - \Bell_{\pol} v = X(\tmat \val - \trew)$. Note that $XU = UX$ from the definition of $U$.

Clearly, when $\pol \in \{0,1\}$ is deterministic the linear programs \eqref{mpr:hybrid_bound_max_cv} and \eqref{mpr:norm_def} are identical. When the policy $\pol$ is stochastic, assume an optimal solution $y$ of \eqref{mpr:norm_def} and let $z = X\tr y$. Then, $z$ is feasible in \eqref{mpr:hybrid_bound_max_cv} with the identical objective value, which shows the inequality.
\end{proof}

\section{Solving Bilinear Programs} \label{sec:solutions}

This section describes methods for solving approximate bilinear programs. Bilinear programs can be easily mapped to other global optimization problems, such as mixed integer linear programs~\citep{Horst1996}.
 We focus on a simple iterative algorithm for solving bilinear programs approximately, which also serves as a basis for many optimal algorithms.

Solving a bilinear program is an NP-complete problem~\citep{Bennett1992}. The membership in NP follows from the finite number of basic feasible solutions of the individual linear programs, each of which can be checked in polynomial time. The NP-hardness is shown by a reduction from the SAT problem.

There are two main approaches to solving bilinear programs optimally. In the first approach, a relaxation of the bilinear program is solved. The solution of the relaxed problem represents a lower bound on the optimal solution. The relaxation is then iteratively refined, for example by adding cutting plane constraints, until the solution becomes feasible. This is a common method used to solve integer linear programs. The relaxation of the bilinear program is typically either a linear or semi-definite program~\citep{Carpara2009}.

In the second approach, feasible, but suboptimal, solutions of the bilinear program are calculated approximately. The approximate algorithms are usually some variation of \aref{alg:simple}. The bilinear program formulation is then refined --- using concavity cuts~\citep{Horst1996} --- to eliminate previously computed feasible solutions and solved again. This procedure can be shown to find the optimal solution by eliminating all suboptimal feasible solutions.

The most common and simplest approximate algorithm for solving bilinear programs is \aref{alg:simple}. This algorithm is shown for the general bilinear program \eqref{mpr:bilinear}, where $f(w,x,y,z)$ represents the objective function. The minimizations in the algorithm are linear programs which can be easily solved. Interestingly, as we will show in \aref{sec:discussion}, \aref{alg:simple} applied to ABP generalizes a version of API.

\begin{algorithm}[t]
    $(x_0,w_0) \leftarrow $ random \;
    $(y_0,z_0) \leftarrow \arg\min_{y,z} f(w_0,x_0,y,z)$ \;
    $i \leftarrow 1$ \;
    \While{$y_{i-1} \neq y_i$ or $x_{i-1} \neq x_i$}{
        $(y_i,z_i) \leftarrow \arg\min_{\{y,z \ss A_2 y + B_2 z = b_2 \; y,z \geq \zero \}} f(w_{i-1},x_{i-1},y,z)$ \;
        $(x_i,w_i) \leftarrow \arg\min_{\{x,w \ss A_1 x + B_1 w = b_1 \; x,w \geq \zero \}} f(w,x,y_i,z_i)$ \;
        $i \leftarrow i + 1$
    }
    \Return $f(w_i,x_i,y_i,z_i)$
\caption{Iterative algorithm for solving \eqref{mpr:bilinear}}\label{alg:simple}
\end{algorithm}

While \aref{alg:simple} is not guaranteed to find an optimal solution, its empirical performance is often remarkably good~\citep{Mangasarian1995}. Its basic properties are summarized by the following proposition.
\begin{prop}[e.g. \citep{Bennett1992}] \label{prop:convergence}
\aref{alg:simple} is guaranteed to converge, assuming that the linear program solutions are in a vertex of the optimality simplex. In addition, the global optimum is a fixed point of the algorithm, and the objective value monotonically improves during execution.
\end{prop}
The proof is based on the finite count of the basic feasible solutions of the individual linear programs. Because the objective function does not increase in any iteration, the algorithm will eventually converge.

\aref{alg:simple} can be further refined in case of approximate bilinear programs. For example, the constraint $\val\in\rep$ in the bilinear programs serves just to simplify the bilinear program and a value function that violates it may still be acceptable. The following proposition motivates the construction of a new value function from two transitive-feasible value functions.
\defprop{\propminimum}{prop:minimum}{
Let $\tilde\val_1$ and $\tilde\val_2$ be feasible value functions in \eqref{mpr:abp_robust}. Then the value function
\[\tilde\val(s) = \min \{ \tilde\val_1(s), \tilde\val_2(s) \}\]
is also feasible in bilinear program \eqref{mpr:abp_robust}.  Therefore $\tilde\val \geq v^*$ and
\[\| v^* - \tilde\val \|_\infty \leq \min \left\{ \| v^* - \tilde\val_1 \|_\infty, \| v^* - \tilde\val_2 \|_\infty  \right\}.\] }
The proof of the proposition is based on Jensen's inequality and is provided in the appendix. Note that $\tilde\val$ may have a greater Bellman residual than either $\tilde\val_1$ or $\tilde\val_2$.

\aref{prop:minimum} can be used to extend \aref{alg:simple} when solving ABPs. One option is to take the state-wise minimum of values from multiple random executions of \aref{alg:simple}, which preserves the transitive feasibility of the value function. However, the increasing number of value functions used to obtain $\tilde\val$ also increases the potential sampling error.

\section{Sampling Guarantees} \label{sec:sampling}

In most practical problems, the number of states is too large to be explicitly enumerated. Even though the value function is restricted to be representable, the problem cannot be solved. The usual approach is to sample a limited number of states, actions, and their transitions to approximately calculate the value function. This section shows basic properties of the samples that can provide guarantees of the solution quality with incomplete samples.

First, we show a formal definition of the samples and then show how to use them. The simplest samples are defined as follows.
\begin{defn}
\emph{One-step simple samples} are defined as:
\[\tilde{\samples} \subseteq \{ (s,a,(s_1 \ldots s_n),\rew(s,a)) \ss s,s' \in \states, \; a \in \actions \},\]
where $s_1 \ldots s_n$ are selected i.i.d. from the distribution
$\tran(s,a)$.
\end{defn}
More informative samples include the full distribution instead of samples from the distribution. While these samples are often unavailable in practice, they are useful in the theoretical analysis of sampling issues.
\begin{defn}
\emph{One-step samples with expectation} are defined as
follows:
\[\bar{\samples} \subseteq \{ (s,a,\tran(s,a),\rew(s,a)) \ss s \in \states, \; a \in \actions \}.\]
\end{defn}
Membership a state in the samples is denoted simply as $s\in\samples$ or $(s,a) \in \samples$ with the remaining variables, such as $\rew(s,a)$ considered to be available implicitly.

The sampling models may vary significantly in different domains. The focus of this work is on problems with either a fixed set of available samples or a domain model. Therefore, we do not analyze methods for gathering samples. We also do not assume that the samples come from previous executions, but rather from a deliberate sample-gathering process.

The samples are used to approximate the Bellman operator and the set of transitive-feasible value functions. \begin{defn} The \emph{sampled Bellman operator} and the corresponding set of sampled transitive-feasible functions are defined as:
\begin{align}
(\bar{L}(\val))(\bar{s}) &= \max_{\{ a \ss (\bar{s},a) \in \bar\samples \}} \rew(\bar{s}, a)  + \disc \sum_{s' \in \states} \tran(\bar{s},a,s')\val(s') \qquad \forall \bar{s}  \in \bar{\samples} \\
\label{eq:upper_second_samples4} \bar{\tf} &= \left\{ \val \ss (\bar{s},a,\tran(\bar{s},a),\rew(\bar{s},a)) \in \bar{\samples},\; \val(\bar{s}) \geq (\bar\Bell \val)(\bar{s}) \right\}
\end{align}
\end{defn}
The less-informative set of samples $\tilde{\samples}$ can be used as follows.
\begin{defn} The \emph{estimated Bellman operator} and the corresponding set of estimated transitive-feasible functions are defined as:
\begin{align}
(\tilde\Bell(\val))(\bar{s}) &= \max_{\{ a \ss (\bar{s},a) \in \tilde\samples \}} \rew(\bar{s}, a) + \disc \frac{1}{n} \sum_{i=1}^n \val(s_i) \qquad \forall \bar{s} \in \tilde{\samples} \\
\label{eq:upper_second_samples3} \tilde{\tf} &= \left\{ \val \ss (\bar{s},a,(s_1 \ldots s_n),\rew(\bar{s},a)) \in \tilde{\samples}, \; \val(\bar{s}) \geq (\tilde{L} \val)(\bar{s}) \right\}
\end{align}
\end{defn}
Notice that operators $\tilde\Bell$ and $\bar\Bell$ map value functions to a subset of all states --- only states that are sampled. The values for other states are assumed to be \emph{undefined}.

The samples can also be used to create an approximation of the initial distribution, or the distribution of visitation-frequencies of a given policy. The estimated initial distribution $\bar\indist$ is defined as:
\[\bar\indist(s) =
\begin{cases}
\indist(s) \qquad& {(s,\cdot,\cdot,\cdot) \in \bar\samples} \\
0 & \text{otherwise}
\end{cases}. \]

The existing sampling bounds for approximate linear programming focus on bounding the probability that a large number of constraints is violated when assuming a distribution over the constraints~\citep{Farias2004}. The difficulty with this approach is that the bounds on the number of violated constraints do not transform easily to the bounds on the quality of the value function, or the policy. In addition, the constraint distribution  is often somewhat arbitrary because it is difficult to define and sampling from the appropriate distributions.

Our approach, on the other hand, is to define properties of the sampled operators that guarantee that the sampling error bounds are small. These bounds do not rely on distributions over constraints and transform directly to bounds on the policy loss. To define bounds on the sampling behavior, we propose the following assumptions.The first assumption limits the error due to missing transitions in the sampled Bellman operator $\bar{\Bell}$.
\begin{asm}[Constraint Sampling Behavior] \label{asm:sampling_behavior}
For all representable value functions $\val \in \rep$:
\[\tf \subseteq \bar{\tf} \subseteq \tf(\epsilon_p)\]
\end{asm}
The second assumption quantifies the error on the estimation of the transitions of the estimated Bellman operator $\tilde\Bell$.
\begin{asm}[Constraint Estimation Behavior] \label{asm:estimation_behavior}
For all representable value functions $\val \in \rep$ the following holds:
\[ \bar{\tf}(-\epsilon_s) \subseteq \tilde{\tf} \subseteq \bar{\tf}(\epsilon_s).\]
\end{asm}
These assumptions are intentionally made generic so that they apply to a wide range of scenarios. Domain specific assumptions are likely to lead to much tighter bounds, but these are beyond the scope of this paper.

Although  we define the sampled Bellman operator directly, in practice only its approximate version is typically estimated. The direct definitions are defined only for the sake of theoretical analysis. The sampled matrices used in bilinear program \eqref{mpr:abp_robust} are defined as follows for all $(s_i, a_j) \in \tilde\samples$.
\begin{align*}
\tilde\tmat \repm &=
\begin{pmatrix}
- & \feats(s_i)\tr - \disc \frac{1}{m} \sum_{s' \in s_1' \ldots s_m'}\tran(s_i,a_j,s') \feats(s')\tr & -  \\
- & \vdots & -
\end{pmatrix}
&\tilde\trew &= \begin{pmatrix} \rew(s_i,a_j) \\ \vdots \end{pmatrix} \\
\tilde{B}(s', (s_i,a_j)) &= \I{ s' = s_i } \quad \forall s' \in \tilde\samples
\end{align*}
The ordering over states in the definitions above is also assumed to be consistent. The sampled version of the bilinear program \eqref{mpr:abp_robust} is then:
\begin{mprog} \label{mpr:abp_robust_s} \tag{s--ABP--$L_\infty$}
\minsep{\pol}{\lambda,\lambda',x}{\pol\tr \lambda + \lambda'}
\stc \tilde{B} \pol = \one &   \tilde\tmat\repm x - \trew \geq \zero
\cs \pol \geq \zero       &   \lambda + \lambda' \one \geq \tilde\tmat\repm x  - \tilde\trew
\cs                       &   \lambda, \lambda' \geq \zero
\end{mprog}
The size of the bilinear program \eqref{mpr:abp_robust_s} scales with the number of samples and features, not with the size of the full MDP, because the variables $\lambda$ and $\pol$ are defined only for state --- action pairs in $\tilde\samples$. That is $|\pol| = |\lambda| = |\{ (s,a) \in \samples \}|$. The number of constraints in \eqref{mpr:abp_robust_s} is approximately three times the number of variables $\lambda$. Finally, the number of variables $x$ corresponds to the number of approximation features.

\aref{thm:abp_loss_robust} shows that sampled robust ABP minimizes $\| \val - \tilde\Bell\val \|_\infty$ or $\| \val - \bar\Bell\val \|_\infty$, depending on the samples used. It is then easy to derive sampling bounds that rely on the sampling assumptions defined above.
\begin{thm} \label{cor:abp_loss_robust_sampled}
Let the optimal solutions to the sampled and precise Bellman residual minimization problems be:
\begin{align*}
v_1 &\in \min_{\val\in\rep\cap\tf} \| \val - \Bell \val \|_\infty  &
v_2 &\in \min_{\val\in\rep\cap\tf} \| \val - \bar\Bell \val \|_\infty &
v_3 &\in \min_{\val\in\rep\cap\tf} \| \val - \tilde\Bell \val \|_\infty
\end{align*}
Value functions $\val_1$, $\val_2$, $\val_3$ correspond to solutions of instances of robust approximate bilinear programs for the given samples. Also let $\hat\val_i = \val_{\pol_i}$, where $\pol_i$ is greedy with respect to $\val_i$. Then, given Assumptions \ref{asm:contains_one}, \ref{asm:sampling_behavior}, and \ref{asm:estimation_behavior}, the following holds:
\begin{align*}
\| \val^* - \hat\val_1 \|_\infty &\leq \frac{2}{1-\disc} \min_{\val\in\rep} \| \val - \Bell \val \|_\infty  \\
\| \val^* - \hat\val_2 \|_\infty &\leq \frac{2}{1-\disc} \left( \min_{\val\in\rep} \| \val - \Bell \val \|_\infty + \epsilon_p \right) \\
\| \val^* - \hat\val_3 \|_\infty &\leq \frac{2}{1-\disc} \left( \min_{\val\in\rep} \| \val - \Bell \val \|_\infty + \epsilon_p + 2 \epsilon_s \right)
\end{align*}
\end{thm}
These bounds show that it is possible to bound policy loss due to incomplete samples. As mentioned above, existing bounds on constraint violation in approximate linear programming~\citep{Farias2004} typically do not easily lead to policy loss bounds.

Sampling guarantees for other bilinear program formulations are very similar. Because they also rely on an approximation of the initial distribution and the policy loss, they require additional assumptions on uniformity of state-samples.
\begin{proof}
We show bounds on $\| \val_i - \Bell \val_i \|_\infty$; the remainder of the theorem follows directly from \aref{thm:abp_loss_robust}. The second inequality follows from \aref{asm:sampling_behavior} and \aref{lem:Bellman_one}, as follows:
\begin{align*}
\val_2 - \Bell \val_2 &\leq \val_2 - \bar\Bell \val_2  \\
&\leq \val_1 - \bar\Bell \val_1 \\
&\leq \val_1 - \Bell \val_1  + \epsilon_p \one
\end{align*}
The second inequality follows from Assumptions \ref{asm:sampling_behavior}, \ref{asm:estimation_behavior} and \aref{lem:Bellman_one}, as follows:
\begin{align*}
\val_3 - \Bell \val_3 &\leq \val_2 - \bar\Bell \val_2  + \epsilon_p \one \\
&\leq \val_2 - \tilde\Bell \val_2  + \epsilon_s \one + \epsilon_p \one \\
&\leq \val_1 - \tilde\Bell \val_1  + \epsilon_s \one + \epsilon_p \one \\
&\leq \val_1 - \Bell \val_1  + 2 \epsilon_s \one + \epsilon_p \one
\end{align*}
Here, we use the fact that $\val_i \geq \Bell \val_i$ and that $v_i$'s minimize the corresponding Bellman residuals.
\end{proof}

To summarize, this section identifies basic assumptions on the sampling behavior and shows that approximate bilinear programming scales well in the face of uncertainty caused by incomplete sampling. More detailed analysis will need to focus on identifying problem-specific assumptions and sampling modes that guarantee the basic conditions, namely satisfying \aref{asm:estimation_behavior} and \aref{asm:sampling_behavior}. Such analysis is beyond the scope of this paper.

\section{Discussion and Related ADP Methods} \label{sec:discussion}

This section describes connections between approximate bilinear programming and two closely related approximate dynamic programming methods: ALP, and \lapi, which are commonly used to solve factored MDPs~\citep{Guestrin2003}. Our analysis sheds light on some of their observed properties and leads to a new \emph{convergent} form of approximate policy iteration.

Approximate bilinear programming addresses some important issues with ALP: 1) ALP provides value function bounds with respect to $L_1$ norm, which does not guarantee small policy loss, 2) ALP's solution quality depends significantly on the heuristically-chosen objective function $\tobj$ in \eqref{mpr:alp}~\citep{Farias2002}, 3) the performance bounds involve a constant $1/(1-\disc)$ which can be very large when $\disc$ is close to 1 and 4) incomplete constraint samples in ALP easily lead to unbounded linear programs. The drawback of using approximate bilinear programming, however, is the higher computational complexity.

The first and the second issue in ALP can be addressed by choosing a problem-specific objective function $\tobj$ ~\citep{Farias2002}. Unfortunately, all existing bounds require that $\tobj$ is chosen based on the optimal ALP solution for $\tobj$. This is impossible to compute in practice. Heuristic values for $\tobj$ are used instead. Robust approximate bilinear program \eqref{mpr:abp_robust}, on the other hand, has no such parameters. On the other hand, the expected-loss bilinear program \eqref{mpr:abp_lone} can be seen as a method for simultaneously optimizing $\tobj$ and the approximate linear program.

The fourth issue in approximate linear programs arises when the constraints  need to be sampled. The ALP may become unbounded with incomplete samples because its objective value is defined using the $L_1$ norm on the value function, and the constraints are defined using the $L_\infty$ norm of the Bellman residual. In approximate bilinear programs, the Bellman residual is used in both the constraints and objective function. The objective function of ABP is then bounded below by 0 for an arbitrarily small number of samples.

The NP-completeness of ABP compares unfavorably with the polynomial complexity of ALP. However, most other approximate dynamic programming algorithms are not guaranteed to converge to a solution in finite time. As we show below, the exponential time complexity of ABP is unavoidable (unless P = NP).

The following theorem shows that the computational complexity of the ABP formulation is asymptotically the same as the complexity of tightly approximating the value function.
\defthm{\thmcomplexity}{thm:complexity}{
Assume $0 < \disc < 1$, and a given $\epsilon > 0$. Then it is NP-complete to determine:
\[\min_{v \in \tf \cap \rep} \| \Bell v - v \|_\infty < \epsilon \qquad \min_{v \in \rep} \| \Bell v - v \|_\infty < \epsilon.\]
The problem remains NP-complete when \aref{asm:contains_one} is satisfied. It is also NP-complete to determine:
\[ \min_{v \in \rep} \| \Bell v - v \|_\infty - \| \val^* - \val \|_{1,\indist} < \epsilon \qquad \min_{v \in \rep} \| \Bell v - v \|_{1,\bar{u}} - \| \val^* - \val \|_{1,\indist} < \epsilon,\]
assuming that $\bar{u}$ is defined as in \aref{rem:constant_u}.}
As the theorem states, the value function approximation does not become computationally simpler even when \aref{asm:contains_one} holds --- a universal assumption in the paper. Notice that ALP can determine whether $\min_{v \in \tf \cap \rep} \| \Bell v - v \|_\infty = 0$ in polynomial time.

The proof of \aref{thm:complexity} is based on a reduction from SAT and can be found in~\aref{sec:complexity}. The policy in the reduction determines the true literal in each clause, and the approximate value function corresponds to the truth value of the literals. The approximation basis forces literals that share the same variable to have consistent values.

Approximate bilinear programming can also improve on API with $L_\infty$ minimization (\lapi~for short), which is a leading method for solving factored MDPs~\citep{Guestrin2003}. Minimizing the $L_\infty$ approximation error is theoretically preferable, since it is compatible with the existing bounds on policy loss~\citep{Guestrin2003}. The bounds on value function approximation in API are typically~\citep{Munos2003}:
\[ \lim\sup_{k\rightarrow\infty} \| v^* - \hat{v}_k \|_{\infty} \leq \frac{2\disc}{(1-\disc)^2} \limsup_{k\rightarrow\infty} \| \tilde{v}_k - \val_k \|_\infty. \]
These bounds are looser than the bounds on solutions of ABP by at least a factor of $1/(1-\disc)$. Often the difference may be up to $1/(1-\disc)^2$ since the error $\|\tilde\val_k - \val_k \|_\infty$ may be significantly larger than $\| \tilde\val_k - \Bell \tilde\val_k\|_\infty$. Finally, the bounds cannot be easily used, because they only hold in the limit.

We propose \emph{Optimistic Approximate Policy Iteration}~(OAPI), a modification of API. OAPI is shown in \aref{alg:api}, where $\appr(\pol)$ is calculated using the following program:
\begin{mprog} \label{mpr:linfty_optimistic}
\minimize{\phi,v} \phi
\stc A v \geq b \quad (\equiv (\eye-\disc P_\pol) v \geq r_\pol \; \; \forall \pol \in \policies)
\cs -(\eye-\disc P_\pol) v + \one \phi \geq - r_\pol
\cs v \in \rep
\end{mprog}
In fact, OAPI corresponds to \aref{alg:simple} applied to ABP because the linear program \eqref{mpr:linfty_optimistic} corresponds to \eqref{mpr:abp_robust} with a fixed $\pol$ (see \eqref{mpr:abp_fixed_alpha}). Then, using \aref{prop:convergence}, we get the following corollary.
\begin{cor}
Optimistic approximate policy iteration converges in finite time. In addition, the Bellman residual of the generated value functions monotonically decreases.
\end{cor}

OAPI differs from \lapi~in two ways: 1) OAPI constrains the Bellman residuals by 0 from below and by $\phi$ from above, and then it minimizes $\phi$. \lapi~constrains the Bellman residuals by $\phi$ from both above and below. 2) OAPI, like API, uses only the current policy for the upper bound on the Bellman residual, but uses \emph{all} the policies for the lower bound on the Bellman residual.

\lapi~cannot return an approximate value function that has a lower Bellman residual than ABP, given the optimality of ABP described in \aref{thm:abp_loss_robust}. However, even OAPI, an approximate ABP algorithm, performs comparably to \lapi, as the following theorem states.
\defthm{\thmoapi}{thm:oapi_the_same}{
Assume that \lapi~converges to a policy $\pol$ and a value function $v$ that both satisfy:
$ \phi = \| v - \Bell_\pol v \|_\infty = \| v - \Bell v \|_\infty $. Then \[\tilde\val = v + \frac{\phi}{1-\disc} \one\]
is feasible in the bilinear program \eqref{mpr:abp_robust}, and it is a fixed point of OAPI.  In addition, the greedy policies with respect to $\tilde\val$ and $v$ are identical.}
Notice that while the optimistic and standard policy iterations can converge to the same solutions, the steps in their computation may not be identical. The actual results will depend on the initialization.

To prove the theorem, we first consider \lapitwo~as a modification of \lapi. \lapitwo~is shown in \aref{alg:api}, where $\appr(\pol)$ is calculated using the following program:
\begin{mprog}\label{mpr:linfty_min_bounded}
\minimize{\phi,v} \phi
\stc (\eye-\disc \tran_a) v + \one \phi \geq \rew_a \quad \forall a\in\actions
\cs -(\eye-\disc \tran_\pol) v + \one \phi \geq -\rew_\pol
\cs \val \in \rep
\end{mprog}
The difference between linear programs \eqref{mpr:linfty_min} and \eqref{mpr:linfty_min_bounded} is that \eqref{mpr:linfty_min} involves only the current policy, while \eqref{mpr:linfty_min_bounded} bounds $(\eye-\disc \tran_a) v + \one \phi \geq \rew_a$ from below for all policies. Linear program \eqref{mpr:linfty_min_bounded} differs from linear program \eqref{mpr:linfty_optimistic} by not bounding the Bellman residual from below by $\zero$.
\begin{prop} \label{prop:greater_thesame}
\lapi~and \lapitwo~generate the same sequence of policies if the initial policies and tie-breaking is the same.
\end{prop}
\begin{proof}
The proposition follows simply by induction from \aref{lem:two_problems}. The basic step follows directly from the assumption. For the inductive step, let $\pol^1_i = \pol^2_i$, where $\pol^1$ and $\pol^2$ are the policies with \eqref{mpr:linfty_min} and \eqref{mpr:linfty_min_bounded}. Then from \aref{lem:two_problems}, we have that the corresponding value functions $v^1_i = v^2_i + c \one$. Because $\pol^1_{i+1}$ and $\pol^2_{i+1}$ are chosen greedily, we have that $\pol^1_{i+1} = \pol^2_{i+1}$.
\end{proof}

The proof of \aref{thm:oapi_the_same} follows.
\begin{proof}
The proof is based on two facts. First, $\tilde\val$ is feasible with respect to the constraints in \eqref{mpr:abp_robust}. The Bellman residual changes for all the policies identically, since a constant vector is added. Second, because $\Bell_\pol$ is greedy with respect to $\tilde\val$, we have that $\tilde\val \geq \Bell_\pol \tilde\val \geq \Bell \tilde\val$. The value function $\tilde\val$ is therefore transitive-feasible.

From \aref{prop:greater_thesame}, \lapi~can be replaced by \lapitwo, which will converge to the same policy $\pol$. \lapitwo~will converge to the value function \[\tilde\val = v + \frac{\phi}{1-\disc} \one.\]
From the constraints in \eqref{mpr:linfty_min_bounded} we have that $\tilde\val \geq \Bell_\pol \tilde\val$.
Then, since $\pol$ is the greedy policy with regard to this value function, we have that $\tilde\val \geq \Bell_\pol \tilde\val \geq \Bell \tilde\val$.
Thus $\tilde\val$ is transitive-feasible and feasible in \eqref{mpr:bilinear} according to \aref{lem:greater_bellman}. The theorem then follows from \aref{lem:two_problems} and from the fact that the greedy policy minimizes the Bellman residual, as in the proof of \aref{lem:policy_error}.
\end{proof}

To summarize, OAPI guarantees convergence, while matching the performance of \lapi. The convergence of OAPI is achieved because given a non-negative Bellman residual, the greedy policy also minimizes the Bellman residual. Because OAPI ensures that the Bellman residual is always non-negative, it can progressively reduce it. In comparison, the greedy policy in \lapi~does not minimize the Bellman residual, and therefore \lapi~does not always reduce it. \aref{thm:oapi_the_same} also explains why API provides better solutions than ALP, as observed in \citep{Guestrin2003}. From the discussion above, ALP can be seen as an $L_1$-norm approximation of a single iteration of OAPI. \lapi, on the other hand, performs many such ALP-like iterations.

\section{Experimental Results} \label{sec:experiments}

In this section, we validate the approach by applying it to simple reinforcement learning benchmark problems. The focus of the paper is on the theoretical properties and the experiments are intentionally designed to avoid interaction between the approximation in the formulation and approximate solution methods.
As \aref{thm:oapi_the_same} shows,  even OAPI, the very simple approximate algorithm for ABP, can perform as well as existing methods on factored MDPs.

ABP is an off-policy approximation method, like LSPI~\citep{Lagoudakis2003} or ALP. That means that the samples can be gathered independently of the control policy. It is necessary, though, that multiple actions are sampled for each state to enable the selection of different policies.

First, we demonstrate and analyze the properties of ABP on a simple chain problem with 200 states, in which the transitions move to the right or left (2 actions) by one step with a centered Gaussian noise of standard deviation 3. The rewards were set to $\sin(i/20)$ for the right action and $\cos(i/20)$ for the left action, where $i$ is the index of the state. This problem is small enough to calculate the optimal value function and to control the approximation features. The approximation basis in this problem is represented by piece-wise linear features, of the form $\phi(s_i)=\pos{i-c}$, for $c$ from 1 to 200. The discount factor in the experiments was $\disc = 0.95$ and the initial distribution was $\indist(130)=1$. We verified that the solutions of the bilinear programs were always close to optimal, albeit suboptimal.

We experimented with the full state-action sample and randomly chose the features. All results are averages over 50 runs with 15 features. In the results, we use ABP to denote a close-to-optimal solution of robust ABP, ABPexp for the bilinear program \eqref{mpr:abp_lone_infty}, and ABPhyb for \eqref{mpr:abp_lonehyb} with $k=5$. API denotes approximate policy iteration that minimizes the $L_2$ norm.

\aref{fig:chain_bellman_residual} shows the Bellman residual attained by the methods. It clearly shows that the robust bilinear formulation most reliably minimizes the Bellman residual. The other two bilinear formulations are not much worse. Notice also the higher standard deviation of ALP and API.
\aref{fig:chain_pl_expected_case} shows the expected policy loss, as specified in \aref{def:policy_loss}, for the calculated value functions. It confirms that the ABP formulation outperforms the robust formulation, since its explicit objective is to minimize the expected loss. Similarly, \aref{fig:chain_pl_worst_case} shows the robust policy loss.  As expected, it confirms the better performance of the robust ABP formulation in this case.

Note that API and ALP may achieve lower policy loss on this particular domain than ABP formulations, even though their Bellman residual is significantly higher. This is possible since ABP simply minimizes bounds on the policy loss. The analysis of tightness of policy loss bounds is beyond the scope of this paper.

\begin{figure}[t]
\centering
\includegraphics[width=0.6\textwidth]{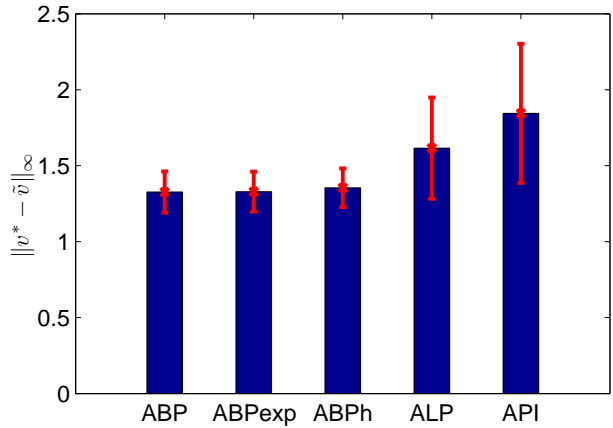}
\caption{$L_\infty$ Bellman residual for the chain problem} \label{fig:chain_bellman_residual}
\end{figure}

\begin{figure}[t]
\centering
\includegraphics[width=0.6\textwidth]{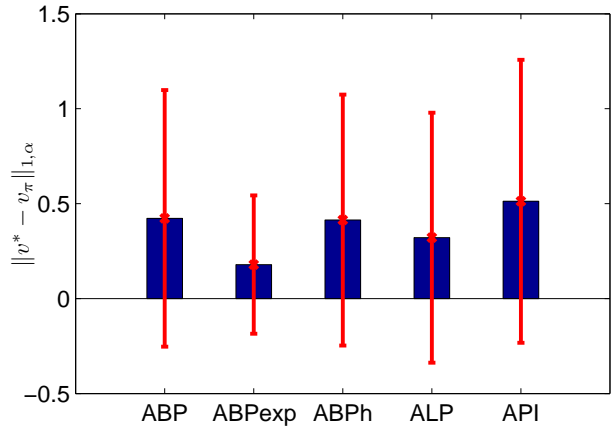}
\caption{Expected policy loss for the chain problem} \label{fig:chain_pl_expected_case}
\end{figure}

\begin{figure}
\centering
\includegraphics[width=0.6\textwidth]{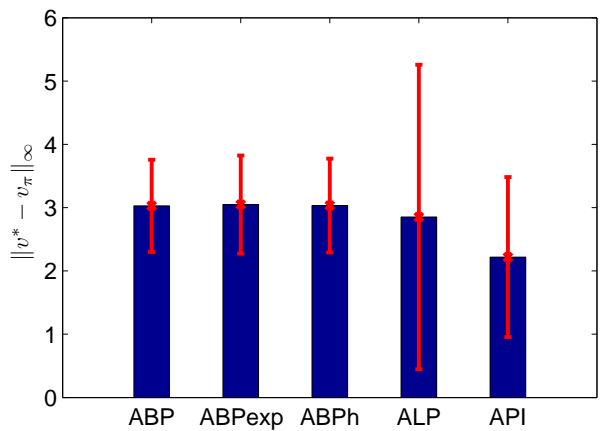}
\caption{Robust policy loss for the chain problem} \label{fig:chain_pl_worst_case}
\end{figure}

In the mountain-car benchmark, an underpowered car needs to climb a hill~\citep{Sutton1998}. To do so, it first needs to back up to an opposite hill to gain sufficient momentum. The car receives a reward of 1 when it climbs the hill. The discount factor in the experiments was $\disc = 0.99$.

\begin{table}
\centering
\subtable[$L_\infty$ error of the Bellman residual]{
\centering
\begin{tabular}{l|rrr}
Features &  \multicolumn{1}{c}{100} & \multicolumn{1}{c}{144} \\
\hline
OAPI &  0.21 (0.23) & 0.13 (0.1) \\
ALP &  13. (13.) & 3.6 (4.3) \\
LSPI &  9. (14.) &  3.9 (7.7) \\
API &  0.46 (0.08) & 0.86 (1.18)
\end{tabular}}
\quad
\subtable[$L_2$ error of the Bellman residual]{
\centering
\begin{tabular}{l|rrr}
Features & \multicolumn{1}{c}{100} & \multicolumn{1}{c}{144} \\
\hline
OAPI &  0.2 (0.3) & 0.1   (1.9) \\
ALP &  9.5 (18.) & 0.3  (0.4) \\
LSPI &  1.2 (1.5) & 0.9  (0.1)  \\
API &  0.04 (0.01) & 0.08 (0.08)
\end{tabular}}
\caption{Bellman residual of the final value function. The values are averages over 5 executions, with the standard deviations shown in parentheses.} \label{tbl:results}
\end{table}

The experiments are designed to determine whether OAPI reliably minimizes the Bellman residual in comparison with API and ALP. We use a uniformly-spaced linear spline to approximate the value function. The constraints were based on 200 uniformly sampled states with all 3 actions per state. We evaluated the methods with the number of the approximation features 100 and 144, which corresponds to the number of linear segments.

The results of robust ABP (in particular OAPI), ALP, API with $L_2$ minimization, and LSPI  are depicted in \aref{tbl:results}. The results are shown for both $L_\infty$ norm and uniformly-weighted $L_2$ norm. The run-times of all these methods are comparable, with ALP being the fastest. Since API (LSPI) is not guaranteed to converge, we ran it for at most 20 iterations, which was an upper bound on the number of iterations of OAPI. The results demonstrate that ABP minimizes the $L_\infty$ Bellman residual much more consistently than the other methods. Note, however, that all the considered algorithms would perform significantly better given a finer approximation.

\section{Conclusion and Future Work} \label{sec:conclusion}

We proposed and analyzed approximate bilinear programming, a new value-function approximation method, which provably minimizes bounds on policy loss. ABP returns the \emph{optimal} approximate value function with respect to the Bellman residual bounds, despite being formulated with regard to transitive-feasible value functions. We also showed that there is no asymptotically simpler formulation, since finding the closest value function and solving a bilinear program are both NP-complete problems. Finally, the formulation leads to the development of OAPI, a new convergent form of API which monotonically improves the objective value function.

While we only discussed approximate solutions of the ABP, a deeper study of bilinear solvers may render optimal solution methods feasible. ABPs have a small number of essential variables (that determine the value function) and a large number of constraints, which can be leveraged by the solvers~\citep{Petrik2007b}. The $L_\infty$ error bound provides good theoretical guarantees, but it may be too conservative in practice. A similar formulation based on $L_2$ norm minimization may be more practical.

We believe that the proposed formulation will help to deepen the understanding of value function approximation and the characteristics of existing solution methods, and potentially lead to the development of more robust and widely-applicable reinforcement learning algorithms.


\newpage

\bibliography{../MainReferences}

\newpage
\appendix

\section{Proofs}

\subsection{Properties of Transitive-Feasible Value Functions}

Basic properties of the Bellman operator, which we often use without a reference are the following.
\begin{lem} \label{lem:Bellman_one}
Let $v$ be any value function and let $c$ be a scalar. Then:
\[ \Bell (v + c \one) = \Bell v + \disc c \one. \]
\end{lem}

\begin{lem}[Monotonicity] \label{lem:Bellman_monotonous}
Let $P$ be a stochastic matrix.  Then both the linear operators $P$ and $(I - \disc P)^{-1}$ are monotonous:
\begin{align*}
x \geq y &\Rightarrow P x \geq P y\\
x \geq y &\Rightarrow (I- \disc P)^{-1} x \geq (I- \disc P)^{-1} y
\end{align*}
for all $x$ and $y$.
\end{lem}

\lemgreatertransitivefeasible
\begin{proof}
Let $P^*$ and $r^*$ be the transition matrix and the reward vector of the policy. Then, we have using \aref{lem:Bellman_monotonous}:
\begin{align*}
v &\geq L v - \epsilon  \one \\
v &\geq \gamma P^* v + r^* - \epsilon  \one \\
(\eye - \gamma P^*) v &\geq r^* - \epsilon  \one \\
v &\geq \invm{\eye - \gamma P^*} r^* - \frac{\epsilon}{1-\gamma}
\end{align*}
\end{proof}

\begin{lem} \label{lem:greater_bellman}
A value function $v$ satisfies $\tmat \val \geq \trew$ if an only if $v \geq \Bell v$. In addition, if $v$ is feasible in \eqref{mpr:abp_robust}, then $v \geq v^*$.
\end{lem}
\begin{proof}
The backward implication of the first part of the lemma follows directly from the definition. The forward implication follows by an existence of $\lambda=\zero$, $\lambda' = \| \pos{A v - r} \|_\infty$, which satisfy the constraints. The constraints on $\pol$ are independent and therefore can be satisfied independently. The second part of the lemma also holds in ALPs~\citep{Farias2002} and is proven identically.
\end{proof}

The minimization $\min_{v \in \rep} \| \Bell v - v \|_\infty$ for a policy $\pol$ can be represented as the following linear program.
\begin{mprog}\label{mpr:minimum_free}
\minimize{\phi,v} \phi
\stc (\eye-\disc P_\pol) v + \one \phi \geq r_\pol
\cs -(\eye-\disc P_\pol) v + \one \phi \geq -r_\pol
\cs v \in \rep
\end{mprog}
Consider also the following linear program.
\begin{mprog} \label{mpr:minimum_transitional}
\minimize{\phi,v} \phi
\stc (\eye - \disc P_\pol) v \geq r_\pol
\cs -(\eye-\disc P_\pol) v + \one \phi \geq - r_\pol
\cs v \in \rep
\end{mprog}
Next we show that the optimal solutions of \eqref{mpr:minimum_free} and \eqref{mpr:minimum_transitional} are closely related.
\begin{lem} \label{lem:two_problems}
Assume \aref{asm:contains_one} and a given policy $\pol$. Let $\phi_1, v_1$ and $\phi_2, v_2$ optimal solutions of linear programs \eqref{mpr:minimum_free} and \eqref{mpr:minimum_transitional} respectively. Define:
\begin{align*}
\bar{v}_1 &= v_1 + \frac{\phi_1}{1-\disc} &\bar{v}_2 &= v_1 - \frac{\phi_2}{2(1-\disc)} \one
\end{align*}
Then:
\begin{enumerate}
    \item $2 \phi_1 = \phi_2$
    \item $\bar{v}_1$ is an optimal solution in \eqref{mpr:minimum_transitional}.
    \item $\bar{v}_2$ is an optimal solution in \eqref{mpr:minimum_free}.
    \item Greedy policies with respect to $v_1$ and $\bar{v}_1$ are identical.
    \item Greedy policies with respect to $v_2$ and $\bar{v}_2$ are identical.
\end{enumerate}
\end{lem}
\begin{proof}
Let $\bar{\phi}_1 = 2 \phi_1$ and $\bar{\phi}_2 = \frac{\phi_2}{2}$. We first show  $\bar{\phi_1}, \bar{v}_1$ is feasible in \eqref{mpr:minimum_transitional}. It is representable since $\one \in \rep$ and it is feasible by the following simple algebraic manipulation:
\begin{align*}
(\eye - \disc P_\pol) \bar{v}_1 &= (\eye - \disc P_\pol) v_1 + (\eye - \disc P_\pol) \frac{\phi_1}{1-\disc} \one  \\
&= (\eye - \disc P_\pol) v_1 + \phi_1 \one \\
&\geq - \phi_1 \one + r_\pol + \phi_1 \one \\
&= r_\pol
\end{align*}
and
\begin{align*}
-(\eye - \disc P_\pol) \bar{v}_1 + \bar{\phi_1} \one &= -(\eye - \disc P_\pol) \bar{v}_1 + 2 \phi_1 \one \\
&=-(\eye - \disc P_\pol) v_1 - (\eye - \disc P_\pol) \frac{\phi_1}{1-\disc} \one  + 2 \phi_1 \one\\
&= -(\eye - \disc P_\pol) v_1 - \phi_1 \one + 2 \phi_1 \one \\
&\geq - \phi_1 \one - r_\pol + 2 \phi_1 \one \\
&= -r_\pol
\end{align*}

Next we show that $\bar{\phi_2}, \bar{v}_2$ is feasible in \eqref{mpr:minimum_free}. This solution is representable, since  $\one \in \rep$, and it is feasible by the following simple algebraic manipulation:
\begin{align*}
(\eye - \disc P_\pol) \bar{v}_2 + \bar{\phi}_2 \one &= (\eye - \disc P_\pol) v_2 - (\eye - \disc P_\pol) \frac{\phi_2}{2(1-\disc)} \one + \frac{\phi_2}{2} \one  \\
&= (\eye - \disc P_\pol) v_2 - \frac{\phi_2}{2} \one  + \frac{1}{2} \bar{\phi}_2 \one  \\
&= (\eye- \disc P_\pol) v_2 \\
&\geq r_\pol
\end{align*}
and
\begin{align*}
-(\eye- \disc P_\pol) \bar{v}_2 + \bar{\phi}_2 \one  &= -(\eye - \disc P_\pol) \bar{v}_2 + \frac{\phi_2}{2} \one\\
&= -(\eye- \disc P_\pol) v_2 - (\eye- \disc P_\pol) \frac{\phi_2}{1-\disc} \one + \frac{\phi_2}{2} \one \\
&= -(\eye - \disc P_\pol) v_2 + \phi_2 \one  \\
&\geq -r_\pol
\end{align*}

It is now easy to shows that $\bar{\phi}_1, \bar{v}_1$ is optimal by contradiction. Assume that there exists a solution $\phi_2 < \bar{\phi}_1$. But then:
\[ 2 \bar{\phi}_2 \leq \phi_2 < \bar{\phi}_1 \leq 2 \phi_1, \]
which is a contradiction with the optimality of $\phi_1$. The optimality of $\bar{\phi}_2, \bar{v}_2$ can be shown similarly.
\end{proof}

\begin{prop} \label{prop:transitive_nonissue}
 \aref{asm:contains_one} implies that:
\[\min_{v\in\rep\cap\tf} \| \Bell v - v\|_\infty \leq 2 \min_{v\in\rep} \| \Bell v - v\|_\infty.\]
\end{prop}
\begin{proof}
Let $\hat{v}$ be the minimizer of $\hat{\phi} = \min_{v\in\rep} \| \Bell v - v\|_\infty$, and let
$\hat{\pol}$ be a policy that is greedy with respect to $\hat{v}$. Define:
\[ \tilde\val = \hat{v} + \frac{\hat{\phi}}{1-\disc}. \]
Then from \aref{lem:two_problems}:
\begin{enumerate}
    \item Value function $\tilde\val$ is an optimal solution of \eqref{mpr:minimum_transitional}:
        $\tilde\val \geq \Bell_\pol \tilde\val$
    \item Policy $\hat{\pol}$ is greedy with regard to $\tilde\val$: $\Bell_{\hat{\pol}} \tilde\val \geq \Bell \tilde\val$
    \item $\|\Bell_\pol \tilde\val - \tilde\val  \|_\infty = 2 \hat{\phi}$
\end{enumerate}
Then using a simple algebraic manipulation:
\[\tilde\val \geq \Bell_{\hat{\pol}} \tilde\val = \Bell \tilde\val \]
and the proposition follows from \aref{lem:greater_bellman}.
\end{proof}

\begin{prop} \label{prop:lower_error}
Let $\tilde\val$ be a solution of the approximate bilinear program \eqref{mpr:abp_robust} and let:
\[ v' = v - \frac{1/2}{(1-\disc)} \|L v - v\|_\infty \one. \]
Then:
\begin{enumerate}
    \item $\| \Bell v' - v'  \|_\infty = \frac{\|L v - v\|_\infty}{2}$.
    \item Greedy policies with respect to $v$ and $v'$ are identical.
\end{enumerate}
\end{prop}
The proposition follows directly from \aref{lem:two_problems}.

\begin{prop} \label{lem:optimal_value_error}
\aref{asm:contains_one} implies that:
\[ \min_{v\in\rep} \| \Bell v - v \|_\infty \leq (1+\disc) \min_{v \in \rep} \|v - v^* \|_\infty .\]
\end{prop}
\begin{proof}
Assume that $\hat{v}$ is the minimizer of  $\min_{v \in \rep} \| v - v^* \|_\infty \leq \epsilon$. Then:
\[ \begin{array}{>{\displaystyle}r>{\displaystyle}c>{\displaystyle}l}
v^* - \epsilon  \one \leq& v  &\leq v^* + \epsilon \one \\
L v^* - \disc \epsilon \one  \leq& \Bell v  &\leq \Bell v^* + \disc \epsilon \one \\
L v^* - \disc \epsilon \one - v \leq& \Bell v - v &\leq \Bell v^* + \disc \epsilon \one  - v\\
L v^* - v^* - (1 + \disc) \epsilon \one  \leq& \Bell v - v &\leq \Bell v^* -v^* + (1+\disc) \epsilon \one \\
- (1 + \disc) \epsilon \one  \leq& \Bell v - v &\leq  (1+\disc) \epsilon \one.
\end{array} \]
\end{proof}

\propminimum
\begin{proof}
Consider a state $s$ and action $a$. Then from transitive feasibility of the value functions $\tilde\val_1$ and $\tilde\val_2$ we have:
\begin{align*}
\tilde\val_1(s) &\geq \disc \sum_{s'\in\states} P(s',a,a) \tilde\val_1(s') + r(s,a) \\
\tilde\val_2(s) &\geq \disc \sum_{s'\in\states} P(s',a,a) \tilde\val_2(s') + r(s,a).
\end{align*}
From the convexity of the $\min$ operator we have that:
\[ \min \left\{\sum_{s'\in\states} P(s',a,a) \tilde\val_1, \sum_{s'\in\states} P(s',a,a) \tilde\val_2(s') \right\} \geq \sum_{s'\in\states} P(s',a,a) \min\{ \tilde\val_1(s'), \tilde\val_2)(s') \}. \]
Then the proposition follows by the following simple algebraic manipulation:
\begin{align*}
\tilde\val = \min \{ \tilde\val_1(s), \tilde\val_2(s) \} &\geq \disc \min \left\{ \sum_{s'\in\states} P(s',a,a) \tilde\val_1, \sum_{s'\in\states} P(s',a,a) \tilde\val_2(s') \right\} + r(s,a) \\
&\geq \disc \sum_{s'\in\states} P(s',a,a) \min\{ \tilde\val_1(s'), \tilde\val_2)(s') \} + r(s,a) \\
&= \disc \sum_{s'\in\states} P(s',a,a) \tilde\val(s) + r(s,a).
\end{align*}
\end{proof}

\begin{lem} \label{lem:dual_sum}
Let $u_\pol$ be the state-action visitation frequency of policy $\pol$. Then:
\[ \one\tr u = \frac{1}{1-\disc}. \]
\end{lem}
\begin{proof}
Let $u_a(s) = u_\pol(s,\pol(s,a))$ for all states $s\in\states$ and actions $a\in\actions$. The lemma follows as:
\begin{align*}
\sum_{a\in\actions} u_a\tr (\eye - \disc \tran_a) &= \tobj\tr \\
\sum_{a\in\actions} u_a\tr (\eye - \disc \tran_a) \one &= \tobj\tr \one \\
(1 - \disc) \sum_{a\in\actions} u_a\tr \one &= 1 \\
u\tr \one &= \frac{1}{1-\disc}. \\
\end{align*}
\end{proof}

\subsection{NP-Completeness} \label{sec:complexity}

\begin{prop}[e.g. \citep{Mangasarian1995}] \label{prop:bilinear_polynomial}
A bilinear program can be solved in NP time.
\end{prop}
There is an optimal solution of the bilinear program such that the solutions of the individual linear programs are basic feasible. The set of all basic feasible solutions is finite, because the feasible regions of $w,x$ and $y,z$ are independent. The value of a basic feasible solution can be calculated in polynomial time.

\thmcomplexity
\begin{proof}
The membership in NP follows from \aref{thm:abp_loss_robust} and \aref{prop:bilinear_polynomial}. We show NP-hardness by  a reduction from the 3SAT problem. We first don't assume \aref{asm:contains_one}. We show the theorem for $\epsilon = 1$. The appropriate $\epsilon$ can be obtained by  simply scaling the rewards in the MDP.

The main idea is to construct an MDP and an approximation basis, such that the approximation error is small whenever the SAT is satisfiable. The value of the states will correspond to the truth value of the literals and clauses. The approximation features $\feats$ will be used to constraint the values of literals that share the same variable. The MDP constructed from the SAT formula is depicted in \aref{fig:reduction}.

\begin{figure}
\centering
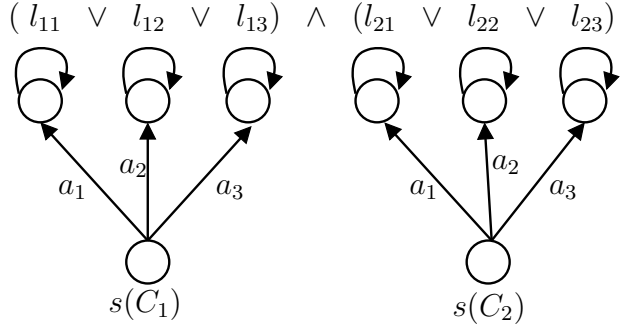
\caption{MDP constructed from the corresponding SAT formula.} \label{fig:reduction}
\end{figure}

Consider a SAT problem with clauses $C_i$:
\[ \bigwedge_{\fromint{i}{1}{n}} C_i = \bigwedge_{\fromint{i}{1}{n}} \left(l_{i 1} \vee l_{i 2} \vee l_{i 3} \right),\]
where $l_{i j}$ are literals. A literal is a variable or a negation of a variable. The variables in the SAT are $x_1 \ldots x_m$. The corresponding MDP is constructed as follows. It has one state $s(l_{i j})$ for every literal $l_{i j}$, one state $s(C_i)$ for each clause $C_i$ and an additional state $\bar{s}$. That is:
\[ \states = \{s(C_i) \ss \fromint{i}{1}{n} \} \cup \{ s(l_{i j}) \ss \fromint{i}{1}{n}, \fromint{j}{1}{3} \} \cup \{\bar{s}\}.\]
There are 3 actions available for each state $s(C_i)$, which determine the true literal of the clause. There is only a single action available in states $s(l_{i j})$ and $\bar{s}$. All transitions in the MDP are deterministic. The transition $t(s,a) = (s',r)$ is from the state $s$ to $s'$, when action $a$ is taken, and the reward received is $r$. The transitions are the following:
\begin{align}
\label{trn:choice} t(s(C_i), a_j) &= \left( s( l_{i j} ), 1-\disc \right) \\
\label{trn:literal} t(s(l_{i j}), a) &= \left(s(l_{i j}),-(1-\disc) \right) \\
\label{trn:bound} t(\bar{s},a) &= (\bar{s}, 2-\disc)
\end{align}
Notice that the rewards depend on the discount factor $\disc$, for notational convenience.

There is one approximation feature for every variable $x_k$ such that:
\begin{align*}
\feats_k(s(C_i)) &= 0 \\
\feats_k(\bar{s}) &= 0 \\
\feats_k(s(l_{i j})) &=
    \begin{cases}
     1 & \cond  l_{i j} = x_k \\
    -1 & \cond  l_{i j} = \neg x_k
    \end{cases}
\end{align*}
An additional feature in the problem $\bar{\feats}$ is defined as:
\begin{align*}
\bar{\feats} (s(C_i)) &= 1 \\
\bar{\feats} (s(l_{i j})) &= 0 \\
\bar{\feats} (\bar{s}) &= 1. \\
\end{align*}
The purpose of state $\bar{s}$ is to ensure that $v(s(c_i)) \geq 2-\disc$, as we assume in the remainder of the proof.

First, we show that if the SAT problem is satisfiable, then $\min_{v \in \rep \cap \tf} \| \Bell  v - v \|_\infty < 1$. The value function $\tilde\val \in \tf$ is constructed as a linear sum of the features as: $v = \repm y$, where $y = (y_1, \ldots, y_m, \bar{y})$. Here $y_k$ corresponds to $\feats_k$ and $\bar{y}$ corresponds to $\bar{\feats}$. The coefficients $y_k$ are constructed from the truth value of the variables as follows:
\begin{align*}
y_k &=
    \begin{cases}
    \disc & \cond x_k = \true \\
     - \disc & \cond x_k = \false
    \end{cases} \\
\bar{y} &= 2-\disc.
\end{align*}
Now define the \emph{deterministic} policy $\pol$ as:
\[ \pol(s(C_i)) = a_j \textrm{ where } l_{i j} = \true .\]
The true literals are guaranteed to exist from the satisfiability. This policy is greedy with respect to $\tilde\val$ and satisfies that $\| \Bell_\pol \tilde\val - \tilde\val \|_\infty \leq 1-\disc^2 $.

The Bellman residuals for all actions and states, given a value function $v$, are defined as:
\[ v(s) - \disc v(s') - r, \]
where $t(s,a) = (s',r)$. Given the value function $\tilde\val$, the residual values are:
\begin{align*}
t(s(C_i), a_j) &= \left( s( l_{i j} ), 1-\disc \right)  : &\quad&
    \begin{cases}
    2 - \disc - \disc^2 + (1-\disc) = 1-\disc^2 &\cond \quad l_{i j} = \true \\
    2 - \disc + \disc^2 + (1-\disc) = 1+\disc^2 &\cond \quad l_{i j} = \false
    \end{cases} \\
t(s(l_{i j}), a) &= \left( s(l_{i j}), (1-\disc) \right) : &\quad&
    \begin{cases}
    \disc - \disc^2 + 1  - \disc = 1 - \disc^2 & \cond l_{i j} = \true  \\
    - \disc + \disc^2 + 1  - \disc = (1 - \disc)^2 > 0 & \cond l_{i j} = \false
    \end{cases} \\
t(\bar{s},a) &= (\bar{s}, 1-\disc) : &\quad& (1-\disc) + \disc - 1 = 0 \\
\end{align*}
It is now clear that $\pol$ is greedy and that:
\[ \| \Bell \tilde\val - \tilde\val \|_\infty = 1 - \disc^2 < 1. \]

We now show that if the SAT problem is not satisfiable then $\min_{v \in \tf \cap \rep} \| \Bell v - v \|_\infty \geq 1-\frac{\disc^2}{2} $. Now, given a value function $v$, there are two possible cases for each $v(s(l_{ij}))$: 1) a positive value, 2) a non-positive value. Two literals that share the same variable will have the same sign, since there is only one feature per each variable.

Assume now that there is a value function $\tilde\val$. There are two possible cases we analyze: 1) all transitions of a greedy policy are to states with positive value, and 2) there is at least one transition to a state with a non-positive value. In the first case, we have that
\[\forall i \, \exists j,\; \tilde\val(s(l_{ij})) > 0. \]
That is, there is a function $q(i)$, which returns the positive literal for the clause $j$. Now, create a satisfiable assignment of the SAT as follows:
\[
x_k =
\begin{cases}
\true  & \cond l_{i q(i)} = x_k \\
\false & \cond l_{i q(i)} = \neg x_k
\end{cases}, \]
with other variables assigned arbitrary values. Given this assignment, all literals with states that have a positive value will be also positive. Since every clause contains at least one positive literal, the SAT is satisfiable, which is a contradiction with the assumption. Therefore, there is at least one transition to a state with a non-positive value.

Let $C_1$ represent the clause with a transition to a literal $l_{11}$ with a non-positive value, without loss of generality. The Bellman residuals at the transitions from these states will be:
\begin{align*}
b_1 &= \tilde\val(s(l_{11})) - \disc \tilde\val(s(l_{11})) + (1-\disc) \geq 0 - 0 + (1-\disc) = 1-\disc \\
b_1 &= \tilde\val(s(C_1)) - \disc \tilde\val(s(l_{11})) - (1-\disc) \geq 2 - \disc - 0 -1 + \disc = 1
\end{align*}
Therefore, the Bellman residual $\tilde\val$ is bounded as:
\[ \| \Bell \tilde\val - \tilde\val \|_\infty \geq \max\{b_1,b_2\} \geq 1. \]
Since we did not make any assumptions on $\tilde\val$, the claim holds for all representable and transitive-feasible value functions. Therefore, $\min_{v\in\rep\cap\tf}\| \Bell v - v \|_\infty \leq 1-\disc^2$ is and only if the 3-SAT problem is feasible.

It remains to show that the problem remains NP-complete even when \aref{asm:contains_one} holds. Define a new state $\bar{s}_1$ with the following transition:
\[t(\bar{s}_2,a) = (\bar{s}_2, -\frac{\disc}{2}).\]
All previously introduced features $\feats$ are zero on the new state. That is $\feats_k(\bar{s}_1) = \bar{\feats}(\bar{s}_1) = 0$. The new constant feature is: $\hat{\feats}(s) = 1$ for all states $s \in \states$, and the matching coefficient is denoted as $\bar{y}_1$. When the formula is satisfiable, then clearly $\min_{v\in \rep\cap\tf} \| \Bell v - v\|_\infty \leq 1-\disc^2$ since the basis is now richer and the Bellman error on the new transition is less than $1-\disc^2$ when $\bar{y}_1 = 0$.

Now we show that when the formula is not satisfiable, then:
\[ \min_{v\in \rep\cap\tf}\| \Bell v - v\|_\infty \geq 1-\frac{\disc^2}{2}. \]
This can be scaled to an appropriate $\epsilon$ by scaling the rewards. Notice that \[0 \leq \bar{y}_1 \leq \frac{\disc}{2}.\]
When $\bar{y}_1 < 0$, the Bellman residual on transitions $s(C_i) \rightarrow s(l_{ij})$ may be decreased by increasing $\bar{y}_1$ while adjusting other coefficients to ensure that $v(s(C_i)) = 2-\disc$. When $\bar{y}_1 > \frac{\disc}{2}$ then the Bellman residual from the state $\bar{s}_1$ is greater than $1-\frac{\disc^2}{2}$. Given the bounds on $\bar{y}_1$, the argument for $y_k = 0$ holds and the minimal Bellman residual is achieved when:
\begin{align*}
v(s(C_i)) - \disc v(s(l_{i j})) - (1-\disc)  &= v(s(\bar{s}_1)) - \disc v(s(\bar{s}_1)) + \frac{\disc}{2} \\
2-\disc  - \disc \bar{y}_1  - (1-\disc)  &= \bar{y}_1 - \disc \bar{y}_1 + \frac{\disc}{2} \\
\bar{y}_1 &= \frac{\disc}{2}.
\end{align*}
Therefore, when the SAT is unsatisfiable, the Bellman residual is at least $1-\frac{\disc^2}{2}$.

The NP-completeness of $\min_{v \in \rep} \| \Bell v - v \|_\infty < \epsilon$ follows trivially from \aref{prop:transitive_nonissue}. The proof for $\| \val - \Bell \val\|_\infty - \indist\tr\val$ is almost identical. The difference is a new state $\hat{s}$, such that $\feats(\hat{s}) = \one$ and $\indist(\hat{s}) = 1$. In that case $\indist\tr\val = 1$ for all $\val\in\rep$. The additional term thus has no effect on the optimization.

The proof can be similarly extended to the minimization of $\| \val - \Bell\val \|_{1,\bar{u}}$. Define $\bar{u}(C_i) = 1/n$  and $\bar{u}(l_{i j}) = 0$. Then the SAT problem is satisfiable if an only if $\| \val - \Bell\val \|_{1,\bar{u}} = 1 - \disc^2$. Note that $\bar{u}$, as defined above, is not an upper bound on the visitation frequencies $u_\pol$. It is likely that the proof could be extended to cover the case $\bar{u}\geq u_\pol$ by more carefully designing the transitions from $C_i$. In particular, there needs to be high probability of returning to $C_i$ and $\bar{u}(l_{i j} > 0$.
\end{proof}

\subsection{Equivalence of OAPI and API} \label{sec:api_proof}

We first consider \lapitwo~as a modification of \lapi.  \lapitwo~is shown in \aref{alg:api}, where $f(\pol)$ is calculated using the following program:
\begin{mprog}\label{mpr:linfty_min_bounded}
\minimize{\phi,v} \phi
\stc (\eye-\gamma P_\pol) v + \one \phi \geq r_\pol
\cs -(\eye-\gamma P_\pol) v + \one \phi \geq -r_\pol
\cs v \in \rep
\end{mprog}

\begin{prop} \label{prop:greater_thesame}
\lapi~and \lapitwo~generate the same sequence of policies if the initial policies and tie-breaking is the same.
\end{prop}
\begin{proof}
The proposition follows simply by induction from \aref{lem:two_problems}. The basic step follows directly from the assumption. For the inductive step, let $\pol^1_i = \pol^2_i$, where $\pol^1$ and $\pol^2$ are the policies with \eqref{mpr:linfty_min} and \eqref{mpr:linfty_min_bounded}. Then from \aref{lem:two_problems}, we have that the corresponding value functions $v^1_i = v^2_i + c \one$. Because $\pol^1_{i+1}$ and $\pol^2_{i+1}$ are chosen greedily, we have that $\pol^1_{i+1} = \pol^2_{i+1}$.
\end{proof}

We are ready now to prove the theorem.
\thmoapi
\begin{proof}
From \aref{prop:greater_thesame}, \lapi~can be replaced by \lapitwo, which will converge to the same policy $\pol$. \lapitwo~will converge to the value function \[\tilde{v} = v + \frac{\phi}{1-\gamma} \one.\]
From the constraints in \eqref{mpr:linfty_min_bounded} we have that:
\[ \tilde{v} \geq L_\pol \tilde{v} .\]
Then, since $\pol$ is the greedy policy with regard to this value function, we have that:
\[ \tilde{v} \geq L_\pol \tilde{v} \geq L \tilde{v}.\]
Thus $\tilde{v}$ is transitive-feasible and feasible in \eqref{mpr:bilinear} according to \aref{lem:greater_bellman}. The theorem then follows from \aref{lem:two_problems} and from the fact that the greedy policy minimizes the Bellman residual, as in the proof of \aref{lem:policy_error}.
\end{proof}

Notice that while the optimistic and standard policy iterations can converge to the same solutions, the steps in their computation may not be identical. The actual results will depend on the initialization.

\end{document}